%% file: main.tex
\documentclass[11pt]{article}
\ifdefined\XeTeXversion
  \AtBeginDocument{}
\else
\fi

\usepackage[margin=1in]{geometry}
\usepackage{amsmath,amssymb,amsthm}
\usepackage{bm}
\usepackage[nopatch=footnote]{microtype}
\usepackage{placeins}
\usepackage{graphicx}
\usepackage{algorithm}
\usepackage{algpseudocode}
\usepackage{booktabs}
\usepackage{needspace}
\usepackage{natbib}
\usepackage[hidelinks]{hyperref}
\usepackage{setspace}
\graphicspath{{figures/}}

\hypersetup{
  pdftitle={Measuring Variable Importance Using Proportional Inclusive Marginal Effects},
  pdfauthor={Youngseo Cho and V. Roshan Joseph},
  pdfkeywords={variable importance, proportional marginal effects, global sensitivity analysis, dimension reduction, nonparametric estimation, computational efficiency}
}

\newcommand{\E}{\mathbb{E}}
\newcommand{\Var}{\operatorname{Var}}
\newcommand{\PV}{\operatorname{PV}}
\newcommand{\PME}{\operatorname{PME}}
\newcommand{\bX}{\mathbf{X}}
\newcommand{\by}{\mathbf{y}}
\input{sections/observed_case_config}

\newtheorem{theorem}{Theorem}
\newtheorem{proposition}{Proposition}
\newtheorem{lemma}{Lemma}

\newtheorem{assumption}{Assumption}

\title{Measuring Variable Importance Using \\ Proportional Inclusive Marginal Effects}

\IfFileExists{latexml.sty}{\usepackage{latexml}}{}
\ifdefined\iflatexml\else
  \expandafter\newif\csname iflatexml\endcsname
\fi
\iflatexml
  \usepackage{authblk}
  \author{Youngseo Cho}
  \author{V. Roshan Joseph}
  \makeatletter
  \lx@add@contact[role=currentaddress,label=shared]{H. Milton Stewart School of Industrial and Systems Engineering\\Georgia Institute of Technology, Atlanta, GA, USA}
  \makeatother
\else
\author{Youngseo Cho \qquad V. Roshan Joseph\footnote{roshan@gatech.edu}\\[0.3em]\small H. Milton Stewart School of Industrial and Systems Engineering\\\small Georgia Institute of Technology, Atlanta, GA, USA}
\fi

\date{}

\usepackage{etoolbox}
\let\reviewOriginalFigure\figure
\renewcommand{\figure}[1][tbp]{\ifstrequal{#1}{h!}{\reviewOriginalFigure[ht!]}{\reviewOriginalFigure[#1]}}
\let\reviewOriginalTable\table
\renewcommand{\table}[1][tbp]{\ifstrequal{#1}{h!}{\reviewOriginalTable[ht!]}{\reviewOriginalTable[#1]}}

\DeclareMathOperator*{\argmin}{\arg\!\min}

\usepackage{array,tabularx}
\usepackage{float}
\usepackage{tocloft}
\begin{document}
\addtocontents{toc}{\protect\setcounter{tocdepth}{-1}}

\maketitle

\begin{abstract}
\begin{singlespace}
\input{sections/00_abstract}

\end{singlespace}
\end{abstract}

\begin{singlespace}
\end{singlespace}

\section{Introduction}
\label{sec:introduction}
\input{sections/01_introduction}

\section{From Variance-Based Sensitivity to PME}
\label{sec:allocation}
\input{sections/02_allocation}

\section{PRIME}
\label{sec:estimation}
\input{sections/03_estimation}

\section{Global Effects Through Local Linear Regressions}
\label{sec:first}
\input{sections/04_first}

\section{Selection and Estimation}
\label{sec:framework}
\input{sections/05_framework}

\FloatBarrier
\section{Numerical Experiments}
\label{sec:experiments}
\input{sections/06_experiments}

\FloatBarrier
\section{Conclusion}
\label{sec:conclusion}
\input{sections/07_conclusion}

\FloatBarrier
\clearpage
\hypertarget{prime.supplement.start}{}
\pdfbookmark[0]{Supplementary Material}{prime.supplement.outline}
\setcounter{section}{0}
\setcounter{subsection}{0}
\setcounter{equation}{0}
\setcounter{figure}{0}
\setcounter{table}{0}
\setcounter{lemma}{0}
\setcounter{theorem}{0}
\setcounter{proposition}{0}
\setcounter{assumption}{0}
\setcounter{corollary}{0}
\setcounter{remark}{0}
\setcounter{algorithm}{0}
\renewcommand{\thesection}{S\arabic{section}}
\renewcommand{\thesubsection}{\thesection.\arabic{subsection}}
\renewcommand{\thefigure}{S\arabic{figure}}
\renewcommand{\thetable}{S\arabic{table}}
\renewcommand{\theHsection}{supp.\arabic{section}}
\renewcommand{\theHsubsection}{\theHsection.\arabic{subsection}}
\renewcommand{\theHsubsubsection}{\theHsubsection.\arabic{subsubsection}}
\renewcommand{\theHequation}{supp.\arabic{equation}}
\renewcommand{\theHfigure}{supp.\arabic{figure}}
\renewcommand{\theHtable}{supp.\arabic{table}}
\renewcommand{\theHlemma}{supp.\arabic{lemma}}
\renewcommand{\theHtheorem}{supp.\arabic{theorem}}
\renewcommand{\theHproposition}{supp.\arabic{proposition}}
\renewcommand{\theHassumption}{supp.\arabic{assumption}}
\renewcommand{\theHcorollary}{supp.\arabic{corollary}}
\renewcommand{\theHremark}{supp.\arabic{remark}}
\providecommand{\theHalgorithm}{supp.\arabic{algorithm}}
\renewcommand{\theHalgorithm}{supp.\arabic{algorithm}}
\onehalfspacing
\addtocontents{toc}{\protect\setcounter{tocdepth}{2}}
\par
\begingroup
\newpage
\null
\vskip 2em
\begin{center}
  {\LARGE Supplementary Material for\\
  Measuring Variable Importance Using\\
  Proportional Inclusive Marginal Effects\par}
  \vskip 1.5em
  {\large
   \lineskip .5em
   \begin{tabular}[t]{c}
   Youngseo Cho \qquad V. Roshan Joseph\\[0.3em]\small H. Milton Stewart School of Industrial and Systems Engineering\\\small Georgia Institute of Technology
   \end{tabular}\par}
  \vskip 1em
  {\large\par}
\end{center}\par
\vskip 1.5em
\thispagestyle{plain}
\endgroup
\setcounter{footnote}{0}

\noindent The Supplement follows the order of the main paper. Section~S1
contains the mathematical proofs. Section~S2 records the complete experimental
protocol, including the primary eight-coordinate
\ifConcreteMain Concrete\else Auto MPG\fi\ design. Section~S3
reports screening and downstream allocation, estimator ablations,
correlated-proxy robustness, common-game sensitivity, scalability, and concise
observed-data checks.

\begingroup
\singlespacing
\small
\setlength{\cftbeforesecskip}{2pt}
\setlength{\cftbeforesubsecskip}{0pt}
\tableofcontents
\endgroup
\clearpage

\section{Proofs}
\input{sections/supplement_proofs}

\section{Experimental Protocol}
\input{sections/supplement_experiments}

\section{Supporting Numerical Results}
\input{sections/supplement_results}

\begin{singlespace}
\setlength{\bibsep}{2pt}
\FloatBarrier
\bibliographystyle{chicago}
\bibliography{references}
\end{singlespace}
\end{document}

%% file: sections/observed_case_config.tex
\newif\ifConcreteMain
\ConcreteMaintrue

%% file: sections/00_abstract.tex
Shapley-based measures are widely used for measuring variable importance in statistics and machine learning. However, these measures have a serious issue--a nonactive variable can get a positive importance measure when it is correlated with an active variable. Herin et al. recently proposed a measure called proportional marginal effects (PME) that mitigates this spurious effect. However, PME is developed for deterministic models and has high computational cost. In this article, we propose a new measure called proportional inclusive marginal effects (PRIME) that overcomes these two limitations of PME. The key idea in PRIME is to use a novel model-free variable selection technique to identify the active variables and include only them in the computation of the importance measure. Analysis of test functions and real datasets show that PRIME can accurately and reliably  quantify variable importance from noisy input-output data.

%% file: sections/01_introduction.tex
Variable importance is a central problem in statistics: given a dataset consisting of
a set of inputs and an output, which inputs are important for predicting the output, and how should their importance
be quantified?  Familiar measures such as the standardized coefficients or $|t|$-values in linear regression and tree-based scores in random forest models are all model-based: they describe
importance with respect to a particular fitted model. Related work studies reliance within a specified model class \citep{fisher2019all} and algorithm-agnostic population importance \citep{williamson2023framework}. Here we study population measures defined directly by the joint distribution of input and output.

The literature on variable importance measures is huge and spans many diverse fields including statistics, machine learning, and applied mathematics. Classical global sensitivity analysis often assumes a known input-output relationship and specified, independent input distributions \citep{saltelli2008global, sensitivitybook2021}. On the other hand, the problems in statistics and machine learning are more complex because they deal with datasets most often collected using an observational study. Here we know neither the input-output relationship nor the input distributions. To make things worse, the inputs are likely to be correlated in the dataset and the output is noisy. Although we are more interested in the statistical/ML problems, we will first describe some techniques used in the field of global sensitivity analysis because they are closely related to the techniques proposed in this paper.

When the input-output relationship (computer model) is known and when the input distributions are independent, we can decompose the overall variability of the output induced by the input distributions using functional analysis of variance and compute certain variance-based global sensitivity measures \citep{sobol1993sensitivity}. Among the various global sensitivity measures, the total Sobol' index proposed by \cite{homma1996importance} is of particular interest because it can be shown that the total Sobol' index of an input is 0 if and only if the output is independent of this input given the other inputs \citep{hart2018approximation}. With independent inputs, total effects collect all main and interaction variance components involving an input. Under dependence, their interpretation is a conditional deletion loss. Consider a simple model: $y=x_1+x_2$, where $x_1,x_2\overset{iid}{\sim} \mathcal{N}(0,1)$. Then the total Sobol' indices of $x_1$ and $x_2$ are 0.5. However, if $cor\{x_1,x_2\}=1$, then the total Sobol' indices of both are 0. This is because $x_1$ is a constant given $x_2$ and therefore $x_1$ does not introduce any variability in $y$ (similarly for $x_2$). These zeros express conditional redundancy: either input predicts the response exactly once the other is known. Total effects therefore remain useful deletion-loss measures, but need not provide an additive allocation of importance under dependence.

\cite{song2016shapley} suggested that the Shapley effect introduced by \cite{owen2014sobol} can be used as a variable importance measure when inputs are dependent. The Shapley effect is a modification of the Shapley value \citep{shapley1953shap} widely used in economics and game theory that is modified to incorporate variance-based sensitivity measures. For the earlier example, Shapley effects of both $x_1$ and $x_2$ are 0.5 even when they are perfectly correlated. Shapley-based measures such as SHAP \citep{lundberg2017unified} are extremely popular for measuring variable importance in real data analysis. See also SAGE in \cite{covert2020understanding} and SPVIM in \cite{williamson2020spvim}.   However, \cite{verdinelli2024feature} recently pointed out that Shapley effects are also adversely affected by the dependence of input distributions. To see this, consider a modified model: $y=x_1+x_2+0x_3+\ldots +0 x_p$, where all variables are perfectly correlated. Now the Shapley effects of all the variables become $1/p$ and thus tend toward 0 as $p$ increases, spreading importance equally over redundant variables. In this perfectly collinear example, the coefficient support of the written regression function is not identifiable from the joint distribution. \cite{verdinelli2024feature} went on to propose a decorrelated version of Leave Out COvariates (LOCO) \citep{lei2018distribution} as a remedy but noted that it also has several issues. See also \cite{verdinelli2024decorrelated}. Interestingly, the LOCO measure is exactly the same as the unnormalized total Sobol' index under a squared-error loss.

\cite{huang2025first} recently proposed a two-stage version of the Shapley effects: first select the active inputs using FIRST (Factor Importance Ranking and Selection using Total indices). Then, estimate the Shapley effects of only the active inputs and set the Shapley effects of the inactive inputs to 0. Their method is termed as FIRST-SHAP. Selecting a subset removes variables excluded by the screen, but retained weak variables can still receive substantial Shapley importance through dependence with stronger variables. This motivates reconsidering the allocation rule as well as the selection step.

\cite{herin2024proportional} proposed proportional marginal effects (PME) to overcome the above issue with Shapley effects that integrate total Sobol' indices with proportional values used in cooperative game theory \citep{feldman1999proportional}.  For example, consider $y=x_1+x_2+0.001x_3$ with jointly Gaussian, mean-zero, unit-variance inputs, $\operatorname{cor}(x_1,x_3)=0.9$, and $x_2$ independent of $(x_1,x_3)$. The Shapley effect of $x_3$ is $0.2028$, while its PME is $5.0\times10^{-7}$; PME assigns much less importance to the input with coefficient $0.001$ in this nonsingular design. This issue of inactive (or exogenous) inputs receiving large positive Shapley effect when correlated with active (or endogenous) inputs has been known for a long time in the case of linear regression \citep{gromping2007importance}, where the method of \cite{lindeman1980introduction} (LMG) produces measures that correspond to the Shapley values and the method of proportional marginal variance decomposition (PMVD) \citep{feldman2005relative} was developed to overcome the issues with LMG. PME can be viewed as an extension of PMVD to nonlinear and complex \emph{deterministic} models. In this article, we extend PMEs for measuring variable importance from real \emph{noisy} data.

\begin{table}[h!]
\centering
\caption{Analogy among some of the variable importance measures in linear regression, global sensitivity analysis, and real data analysis.}
\begin{tabular}{ c|c|c }
\toprule
Linear regression & Global sensitivity analysis & Nonparametric data analysis\\
\hline
$t^2$ & Total Sobol' indices & LOCO, FIRST\\
LMG & Shapley effects & SHAP, SAGE, SPVIM, FIRST-SHAP\\
PMVD & PME & PRIME (proposed measure)\\
\bottomrule
\end{tabular}
\label{tab:analogy}
\end{table}

The second main contribution of our work is in computation. If there are $p$ inputs, we need to consider $2^p$ submodels for computing PME which can be a daunting task when $p$ is large. Following the idea of FIRST-SHAP \citep{huang2025first}, we will reduce this computational burden using a FIRST-inspired forward--backward screening procedure. If there are $q\ll p$ selected variables, then the final allocation step evaluates $2^q$ submodels. Proposition~\ref{prop:selected_set_equivalence} shows that retaining a mean-sufficient set preserves the population PME target after zero padding. Thus, surplus retained coordinates leave this population allocation unchanged under the stated identifiability condition. We call the resulting procedure \emph{Proportional Inclusive Marginal Effects} (PRIME), reflecting the inclusion of only selected coordinates in the allocation step. Table 1 shows the connections among the different measures in the literature that we have reviewed so far.

A third main contribution of our work is in selection and estimation. FIRST uses the nearest neighbor technique proposed in \cite{broto2020variance} for the estimation of the total Sobol' indices. We found that these estimates can be inaccurate, especially when the sample size is small. We therefore propose a local-linear regression approach that removes affine conditional-mean drift from the risk estimates. These estimates are used in both coordinate selection and PRIME allocation, and their accuracy is evaluated in our experiments.

The article is organized as follows. Section~\ref{sec:allocation} reviews
variance-based sensitivity measures and defines PME.  Section~\ref{sec:estimation}
proposes PRIME.  Section~\ref{sec:first} develops
local-linear risk estimation and  Section~\ref{sec:framework} develops a model-free variable selection technique.  Section~\ref{sec:experiments}
evaluates the proposed method through test functions and  real data
experiments. We conclude with some remarks in   Section~\ref{sec:conclusion}.  Proofs and additional
experimental details are provided in the Supplementary materials.

%% file: sections/02_allocation.tex
In this section, we will review some of the techniques from global sensitivity analysis that are relevant to our proposed method (second column in Table \ref{tab:analogy}).
\subsection{Total Sobol' Indices}

Let \(P=\{1,\ldots,p\}\) be the indices of the inputs
\(X=(X_1,\ldots,X_p)\sim F_X\), and let \(M=m(X)\) be square-integrable.
All expectations and variances in this section are taken under the input distribution \(F_X\).  For
\(S\subseteq P\), define
\[
  m_S(x_S)=\E(M\mid X_S=x_S),
  \qquad
  V(S)=\Var\{m_S(X_S)\},
\]
and the residual risk
\[
  \mathcal L_M(S)
  =
  \inf_{g\in L^2(F_{X_S})}\E\{M-g(X_S)\}^2
  =
  \E\{\Var(M\mid X_S)\},
\]
where the infimum is attained at \(g=m_S\) \citep{hart2018approximation}. The
law of total variance gives
\begin{equation}
  V(S)+\mathcal L_M(S)=\Var(M).
  \label{eq:variance_decomposition}
\end{equation}
Thus \(V(S)\) is the variance of the conditional mean induced by \(X_S\),
and the reduction in prediction risk achieved by approximating \(M\) with
\(m_S(X_S)\) instead of its overall mean.

Assume \(V(P)=\Var(M)>0\).  For any \(A\subseteq P\), the closed and total
Sobol' indices of the model response are
\begin{align}
  S_A^{\mathrm{cl}}
  &=
  \frac{V(A)}{V(P)},
  \label{eq:closed_sobol_effect}\\
  S_A^{\mathrm T}
  &=
  \frac{\E\{\Var(M\mid X_{P\setminus A})\}}{\Var(M)}
  =
  \frac{V(P)-V(P\setminus A)}{V(P)}
  =
  \frac{\mathcal L_M(P\setminus A)}{V(P)}.
  \label{eq:total_sobol_effect}
\end{align}
For a single coordinate, \(S_j=S_{\{j\}}^{\mathrm{cl}}\) is conventionally
called its \emph{first-order Sobol' index} \citep{sobol1993sensitivity} and
\(S_{T_j}=S_{\{j\}}^{\mathrm T}\) its \emph{total Sobol' index}
\citep{homma1996importance}.

\Needspace{5\baselineskip}
Under the independent product law \(F_X=\bigotimes_{j=1}^pF_{X_j}\),
\(M\) has the orthogonal expansion \citep{sobol1993sensitivity}:
\[
  M=\mu+\sum_{\varnothing\ne A\subseteq P} h_A(X_A),
\]
which gives the functional ANOVA decomposition
\[
  V(P)=\sum_{\varnothing\ne A\subseteq P}\nu_A,
\]
where \(\nu_A=\Var\{h_A(X_A)\}\).
Consequently,
\begin{equation}
  S_j=\frac{\nu_{\{j\}}}{V(P)} \;\; \textrm{and}\;\;
  S_{T_j}=\frac{\sum_{A\ni j}\nu_A}{V(P)}.
  \label{eq:first_total_anova}
\end{equation}
Summing over coordinates gives \citep{song2016shapley}
\begin{align}
  \sum_{j=1}^p S_j
  &=
  1-\frac{\sum_{|A|\geq2}\nu_A}{V(P)}
  \leq1,
  \label{eq:first_order_sum}\\
  \sum_{j=1}^p S_{T_j}
  &=
  \frac{\sum_{\varnothing\ne A\subseteq P}|A|\nu_A}{V(P)}
  \geq1,
  \label{eq:total_order_sum}
\end{align}
where $|A|$ denotes the cardinality of the set $A$. First-order indices omit interaction components, whereas total-order indices
count an \(|A|\)-way interaction once for each coordinate in \(A\). Except
for additive responses, neither vector sums to one. Thus, they do not partition the importance properly to the $p$ inputs.

\subsection{Shapley Effects}

Equations~\eqref{eq:first_order_sum}--\eqref{eq:total_order_sum} leave one
allocation question unresolved: how should a shared interaction be divided
among its participating coordinates so that the importance measures of the coordinates will sum to one?  Shapley allocation \citep{shapley1953shap} answers this question
through sequential gains.  For a permutation \(\pi=(\pi_1,\ldots,\pi_p)\) of
the coordinates, reveal them in the order specified by \(\pi\).  The
contribution of \(i\) in that order is the increase in explained variance when
\(i\) is added after its predecessors.  Averaging this contribution over all
orders removes dependence on any particular sequence.

Formally, let \(\Pi(P)\) be the set of all \(p!\) permutations of \(P\), and let
\(\operatorname{Pre}_i(\pi)\) denote the coordinates preceding \(i\) in
\(\pi\).  Applying the Shapley value to the explained-variance set function
\(t(A)=V(A)\) gives \citep{owen2014sobol}
\begin{align}
  \operatorname{Sh}_i
  &=
  \frac1{t(P)p!}\sum_{\pi\in\Pi(P)}
  \left[
    t\{\operatorname{Pre}_i(\pi)\cup\{i\}\}
    -t\{\operatorname{Pre}_i(\pi)\}
  \right]
  \label{eq:shapley_permute}\\
  &=
  \frac{1}{t(P)}\sum_{A\subseteq P\setminus\{i\}}
  \frac{|A|!(p-|A|-1)!}{p!}
  \{t(A\cup\{i\})-t(A)\}.
  \label{eq:shapley}
\end{align}
This averaging solves the allocation deficiency in
\eqref{eq:first_order_sum}--\eqref{eq:total_order_sum}.  As shown in \cite{owen2014sobol} and \cite{song2016shapley},
for any joint distribution of \(X\), 
\begin{equation}
    \sum_{i=1}^p \operatorname{Sh}_i=1
    \quad\text{and}\quad
    \operatorname{Sh}_i\geq0 \quad\forall i.
\end{equation}
Furthermore, if the inputs are
independent, then \citep{borgonovo2024many}
\begin{equation}
  \operatorname{Sh}_i
  =
  \frac1{V(P)}
  \sum_{A\ni i}\frac{\nu_A}{|A|}
  \label{eq:shapley_anova}
\end{equation}
and therefore
$S_i\leq\operatorname{Sh}_i\leq S_{T_i} \quad\forall i$.
Thus a \(k\)-way interaction is divided equally among its \(k\) participants
\citep{owen2014sobol}.  The same monotonicity and telescoping
argument remains valid under dependence, even though an orthogonal ANOVA
interpretation does not, and therefore, Shapley effects can be genuinely used as important measures with dependent inputs \citep{song2016shapley, owen2017shapley, iooss2019shapley}.

Unfortunately, dependence creates another difficulty \citep{herin2024proportional}.  To see this, let
\((X_1,X_2)\) be centered Gaussian with unit variances and correlation
\(\rho\), and let \(M=X_1\).  Then
\[
  t(\{1\})=1,\qquad t(\{2\})=\rho^2,\qquad t(\{1,2\})=1,
\]
so
\begin{equation*}
  \operatorname{Sh}_1=1-\frac{\rho^2}{2},
  \qquad
  \operatorname{Sh}_2=\frac{\rho^2}{2}.
  \label{eq:shapley_proxy_example}
\end{equation*}
Thus, although \(X_2\) is absent in the functional relationship, it receives positive
Shapley effect because of its association with \(X_1\).  Shapley
effects therefore allocate variance across coordinates, including proxies
that predict the response through their association with other inputs, which is undesirable.

\subsection{Proportional Marginal Effects}

To overcome the problems with Shapley effects,
\citet{herin2024proportional} proposed proportional marginal effects (PME),
which replace balanced sharing by the proportional value from cooperative game
theory \citep{feldman1999proportional,ortmann2000proportional}.  For a
nonempty \(A\subseteq P\), let \(\Pi(A)\) be its permutation set and let
\(t:2^A\to[0,\infty)\) be monotone, with \(t(\varnothing)=0\) and
\(t(B)>0\) for every nonempty \(B\subseteq A\).
For \(i\in A\), the proportional value is
\begin{equation}
  \PV_i(A,t)
  =
  \sum_{\pi\in\Pi(A)}p_t(\pi)
  \left[
    t\{\operatorname{Pre}_i(\pi)\cup\{i\}\}
    -t\{\operatorname{Pre}_i(\pi)\}
  \right],
  \label{eq:pme_permute}
\end{equation}
where
\begin{equation}
  p_t(\pi)
  =
  \frac{
    \left[
      \prod_{\ell\in A}
      t\{\operatorname{Pre}_\ell(\pi)\cup\{\ell\}\}
    \right]^{-1}
  }{
    \displaystyle
    \sum_{\sigma\in\Pi(A)}
    \left[
      \prod_{\ell\in A}
      t\{\operatorname{Pre}_\ell(\sigma)\cup\{\ell\}\}
    \right]^{-1}
  }.
  \label{eq:pv_order_weight}
\end{equation}
Compared to \eqref{eq:shapley_permute},
\eqref{eq:pv_order_weight} assigns greater weight to orderings with smaller
products of successive subset values; see \citet{herin2024proportional} for details on
this proportional redistribution principle.

The proportional allocation can also be computed from a ratio potential.
For \(B\subseteq A\), define this potential recursively by
\begin{equation}
  R(\varnothing,t)=1,
  \qquad
  R(B,t)
  =
  t(B)
  \left\{
    \sum_{i\in B}R(B\setminus\{i\},t)^{-1}
  \right\}^{-1},
  \qquad B\neq\varnothing.
  \label{eq:ratio_potential}
\end{equation}
Then
\begin{equation}
  \PV_i(A,t)=\frac{R(A,t)}{R(A\setminus\{i\},t)},
  \qquad i\in A.
  \label{eq:proportional_value}
\end{equation}

Ordinary proportional values require strict positivity, whereas total Sobol'
indices may vanish on nonempty subsets.  PME uses the canonical extension
\(\PV^0\) defined by \citet[Theorem~1]{herin2024proportional}.
For a nonnegative monotone game with \(t(P)>0\), replace every nonempty
zero-valued subset by the same \(\epsilon>0\), leave positive values and
\(t(\varnothing)=0\) unchanged, and denote the resulting game by
\(t^{(\epsilon)}\).  For \(\epsilon\) below the smallest positive subset
value, this game is positive and monotone, and the extension is
\[
  \PV^0(P,t)=\lim_{\epsilon\downarrow0}\PV(P,t^{(\epsilon)}).
\]
This definition specifies the boundary regularization; continuity under
arbitrary perturbations is established in the Supplement for the games used
in our consistency results.
A null coalition is a subset with zero set-function value.  We call
\(E\) the \emph{greatest null coalition} if \(t(E)=0\) and every null
coalition is contained in \(E\).  This is stronger than having a unique
null coalition of maximum cardinality.  When a greatest null coalition
exists, define
\[
  t_E(B)=t(E\cup B),
  \qquad B\subseteq P\setminus E.
\]
The canonical extension then simplifies to
\begin{equation}
  \PV_i^0(P,t)
  =
  \begin{cases}
    0, & i\in E,\\
    \PV_i(P\setminus E,t_E), & i\notin E.
  \end{cases}
  \label{eq:continuous_pv_unique_null}
\end{equation}

Total Sobol' indices provide the required exclusion criterion.  For any
\(E\subseteq P\)
\citep{hart2018approximation},
\begin{equation}
  S_E^{\mathrm T}=0
  \quad\Longleftrightarrow\quad
  V(P\setminus E)=V(P)
  \quad\Longleftrightarrow\quad
  \E\!\left[
    \{m_P(X)-m_{P\setminus E}(X_{P\setminus E})\}^2
  \right]=0.
  \label{eq:total_zero_mean_sufficiency}
\end{equation}
Thus, the zero sets of \(S^{\mathrm T}\) identify coordinates that can be
removed jointly without changing the regression mean.  This equivalence
motivates applying the extended proportional value to \(S^{\mathrm T}\), which gives the proportional marginal effects (PME) vector:
\begin{equation}
  \PME(P)=\PV^0\{P,S^{\mathrm T}\}.
  \label{eq:pme}
\end{equation}

\Needspace{3\baselineskip}
To understand the behavior of PME compared to Shapley effects, consider a slightly modified example from \cite{herin2024proportional}, where the inputs are centered Gaussian with unit
marginal variances, \(\operatorname{Corr}(X_1,X_3)=\rho\), \(|\rho|<1\), and \(X_2\)
independent of \((X_1,X_3)\).  Let
\begin{equation}
  m(X)=X_1+X_2+\beta X_3,
  \qquad 0\leq\beta\leq1,
  \label{eq:gaussian_beta_model}
\end{equation}
and write \(V_\beta=2+\beta^2+2\rho\beta\).  Direct calculation gives
\begin{align}
  \operatorname{Sh}_1(\beta)
  &=
  \frac{1+\rho\beta-\tfrac12\rho^2(1-\beta^2)}{V_\beta},
  &
  \PME_1(\beta)
  &=
  \frac{1+\beta^2+2\rho\beta}
       {(1+\beta^2)V_\beta},
  \notag\\
  \operatorname{Sh}_2(\beta)
  &=\frac1{V_\beta},
  &
  \PME_2(\beta)
  &=\frac1{V_\beta},
  \notag\\
  \operatorname{Sh}_3(\beta)
  &=
  \frac{\beta^2+\rho\beta+\tfrac12\rho^2(1-\beta^2)}{V_\beta},
  &
  \PME_3(\beta)
  &=
  \frac{\beta^2(1+\beta^2+2\rho\beta)}
       {(1+\beta^2)V_\beta}.
  \label{eq:gaussian_beta_allocations}
\end{align}

At \(\beta=0\), \(X_3\) is a correlated proxy and
\[
  \operatorname{Sh}(0)
  =
  \left(\tfrac12-\tfrac{\rho^2}{4},\;
        \tfrac12,\;\tfrac{\rho^2}{4}\right),
  \qquad
  \PME(0)=\left(\tfrac12,\;\tfrac12,\;0\right).
\]
Thus, the proxy Shapley effect increases quadratically with \(|\rho|\), whereas its
PME remains zero irrespective of \(\rho\).  Their behavior for \(\beta>0\) is
shown in Figure~\ref{fig:shapley_pme_proxy}.  For the nonnegative correlations shown, increasing \(\rho\) substantially
inflates \(\operatorname{Sh}_3(\beta)\) near the origin, whereas
\(\PME_3(\beta)\) remains small for small \(\beta\).  Thus omitting a weak
coordinate perturbs the PME allocation continuously near \(\beta=0\), whereas
its Shapley allocation can remain substantial under strong correlation.  This
stability motivates the two-stage procedure proposed in the next section.

\begin{figure}[H]
\centering
\includegraphics[width=0.76\textwidth]{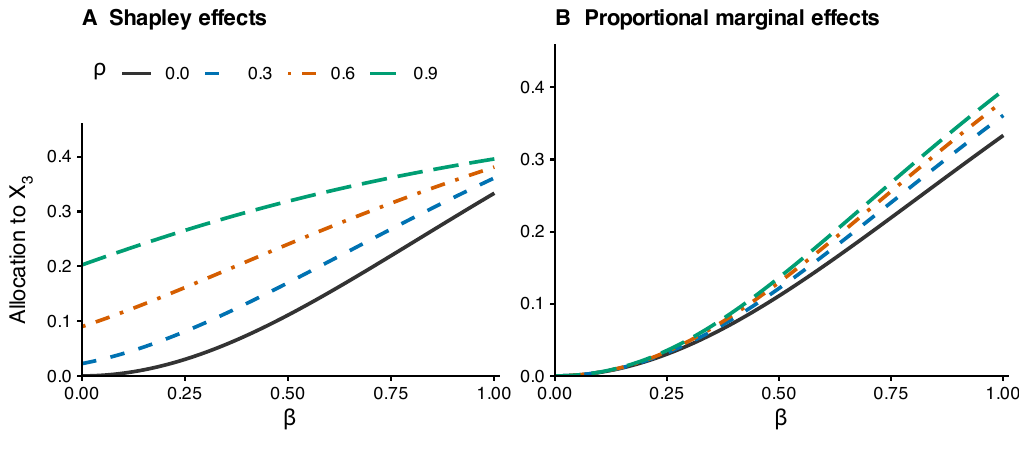}
\caption{Allocation to \(X_3\) in \eqref{eq:gaussian_beta_model}.  Curves show
(A) Shapley effects and (B) PME as the coefficient \(\beta\) increases from 0 to 1 for four correlation settings.}
\label{fig:shapley_pme_proxy}
\end{figure}

%% file: sections/03_estimation.tex
As explained in the previous section, the main advantage of PME over Shapley effects is that it automatically assigns small values to inactive inputs even when they are correlated with active inputs. When the models are deterministic, the inactive inputs can be identified using total Sobol' indices and can be assigned a zero value as in (\ref{eq:continuous_pv_unique_null}). With a finite i.i.d. input-output sample, estimated indices need not vanish even for inactive inputs. Moreover, an unadjusted total-effect calculation based on the observed-response residual variance includes response noise and is positive when that noise variance is positive, even for inactive inputs. Since we are dealing with noisy input-output data with no knowledge about the underlying model or distributions, we need an improved procedure to identify the active set and estimate the PME. Motivated by FIRST-SHAP \citep{huang2025first}, we propose a two-stage estimation procedure for PME:
\begin{enumerate}
    \item Identify the active and inactive inputs using a model-free variable selection procedure.
    \item Estimate the PME for only the active inputs and set the PME of the inactive inputs to be zero.
\end{enumerate}
Since we \emph{include} only the selected inputs in the PME computation, we call the resulting indices proportional inclusive marginal effects (PRIME). We first explain its construction, and then discuss the effect of noise on its estimation. Section~\ref{sec:framework} gives the selection and estimation procedure.

Let
\(\mathcal D_n=\{(X_i,Y_i):1\leq i\leq n\}\) be the observed data consisting of $n$ independent copies of
\((X,Y)\), where
\[
  Y=m(X)+\varepsilon,
  \qquad
  \E(\varepsilon\mid X)=0,
  \qquad
  \E(\varepsilon^2)<\infty.
\]
Call \(S\subseteq P\) mean-sufficient if
\begin{equation}
  \E(Y\mid X_S)=\E(Y\mid X)
  \quad\text{a.s.}
  \label{eq:mean_sufficient}
\end{equation}
Because \(X_S\subseteq X\), the \(L^2\) projection identity gives
\begin{equation}
  V(P)-V(S)
  =
  \E\left[
    \{m_P(X)-m_S(X_S)\}^2
  \right].
  \label{eq:mean_sufficiency_identity_theory}
\end{equation}
The right-hand side measures the prediction error from replacing the full
regression mean by its projection onto the subset of coordinates. It is zero
if and only if \eqref{eq:mean_sufficient} holds; equivalently, \(V(S)=V(P)\).
Let
\[
  \mathcal S_{\mathrm{ms}}
  =
  \{S\subseteq P:V(S)=V(P)\}
\]
be the family of mean-sufficient sets, and assume for identification that it
has a unique inclusion-minimal element \(A\). That is, \(A\) is a mean-sufficient set from which
no coordinate can be removed while preserving the regression mean. For any
\(C\subseteq P\) with \(V(C)>0\), define
\begin{equation}
  T_C(B)=\frac{V(C)-V(C\setminus B)}{V(C)},
  \qquad B\subseteq C.
  \label{eq:selected_total_effect}
\end{equation}
Applying \eqref{eq:pme} with \(C\) as the coordinate index set gives the
selected-set PME vector
\begin{equation}
  \PME(C):=\PV^0(C,T_C)\in\mathbb R^{|C|},
  \qquad
  \PME_j(C):=[\PME(C)]_j.
  \label{eq:selected_set_pme}
\end{equation}
Let
\begin{equation}
  \psi^\star:=\PME(P)\in\mathbb R^p.
  \label{eq:target_allocation}
\end{equation}
The complete PME vector \(\psi(C)\in\mathbb R^p\) is
\begin{equation}
  \psi_j(C)
  =
  \begin{cases}
    \PME_j(C), & j\in C,\\
    0, & j\notin C.
  \end{cases}
  \label{eq:padded_selected_pme}
\end{equation}

\begin{proposition}[Selected-set equivalence]
\label{prop:selected_set_equivalence}
For every \(C\in\mathcal S_{\mathrm{ms}}\), uniqueness implies \(A\subseteq C\),
and \(E=C\setminus A\) satisfies \(T_C(B)=0\) exactly when
\(B\subseteq E\). Thus \(E\) is the greatest null coalition. Moreover,
\begin{equation}
  \psi(C)=\psi^\star.
  \label{eq:mean_sufficient_invariance}
\end{equation}
\end{proposition}

Let \(\widehat A\) be the set selected from the data \(\mathcal D_n\) to
estimate a mean-sufficient set. PRIME estimates the PME using only the
selected coordinates and assigns zero to coordinates outside \(\widehat A\). When \(\widehat A\) is mean-sufficient,
Proposition~\ref{prop:selected_set_equivalence} gives
\(\psi(\widehat A)=\psi^\star\), so the selected-set calculation preserves
the full-coordinate population target. Evaluating subsets of \(\widehat A\)
reduces their number from \(2^p\) to \(2^{|\widehat A|}\).
Section~\ref{sec:framework} gives the selection procedure and defines the
PRIME estimator in \eqref{eq:prime_estimator}.


\label{sec:noisy_subset_risks}

Some adjustments for the PRIME estimates are needed due to the noise in the output data. For \(S\subseteq P\), define the observed-response risk
\begin{equation}
  \mathcal L_Y(S)
  =
  \inf_{g\in L^2(F_{X_S})}\E\{Y-g(X_S)\}^2
  =
  \E\{\Var(Y\mid X_S)\}.
  \label{eq:observed_response_risk}
\end{equation}
Let \(\sigma_\varepsilon^2=\E(\varepsilon^2)\), where the conditional variance
of \(\varepsilon\) may depend on \(X\).

\begin{proposition}[Identification under response noise]
\label{prop:noise_decomposition}
For every \(S\subseteq P\),
\begin{align}
  \mathcal L_Y(S)
  &=
  \mathcal L_M(S)+\sigma_\varepsilon^2,
  \label{eq:noise_decomposition}\\
  V(S)
  &=
  \Var(M)-\mathcal L_M(S)
  =
  \Var(Y)-\mathcal L_Y(S).
  \label{eq:noisy_value_identification}
\end{align}
\end{proposition}

Since $V(S)$ and $\mathcal L_M(S)$ are nonnegative, given the estimates of \(\Var(Y)\), \(\mathcal L_Y(S)\), and
\(\sigma_\varepsilon^2\), we obtain
\begin{align}
\widehat{\mathcal L}_M(S) &=
  \left[\widehat{\mathcal L}_Y(S)-\widehat\sigma_\varepsilon^2\right]_+,
  \label{eq:estimated_signal_risk}\\
  \widehat V(S)
  &=
  \left[\widehat{\Var}(Y)-\widehat{\mathcal L}_Y(S)\right]_+,
  \label{eq:noise_adjusted_value}
\end{align}
where $[x]_+=\max(x,0)$, which enforces the nonnegativity restriction. More details of the estimators are given in the next two sections.

%% file: sections/04_first.tex
Section~\ref{sec:estimation} shows that PRIME can be estimated by finding a
mean-sufficient coordinate set and its subset values \(V(S)\).
The construction can use any selector that consistently retains such a set.
We build on FIRST \citep{huang2025first} because its selection criterion uses
the total Sobol' effects underlying PME.  However, its nearest-neighbor risk estimator
\citep{broto2020variance} can be inaccurate when the sample size \(n\) is
small.  We therefore propose an improved estimator in this section.

\subsection{FIRST}
FIRST ranks and selects coordinates through estimated total Sobol' effects
\citep{huang2025first}.  As mentioned earlier, dependence complicates interpreting these indices as
overall variable importance, but their underlying squared-error deletion
losses remain useful for selection \citep{hart2018approximation}.  FIRST uses
these risk comparisons in one-coordinate-at-a-time forward and backward
searches.  For the regression mean \(M\), the unnormalized total effect for
deleting \(B\subseteq P\) is
\[
  V(P)-V(P\setminus B)
  =
  \mathcal L_M(P\setminus B)-\mathcal L_M(P)
  =
  \mathcal L_Y(P\setminus B)-\mathcal L_Y(P).
\]
Dividing this loss by \(V(P)\) gives the normalized total Sobol' index
\(S_B^{\mathrm T}\).
The quantity $V(S)$ can be estimated from (\ref{eq:noise_adjusted_value}), where $\widehat{\mathcal L}_Y(S)$ can be obtained using a nearest-neighbor local-constant estimator based on
\citet{broto2020variance}:
\begin{equation}
  \widehat{\mathcal L}_{\mathrm{LC}}(S)
  =
  \frac1n\sum_{i=1}^n
  \frac{1}{K-1}
  \sum_{j\in\mathcal N(i;S)}
  \{Y_j-\overline Y_{i,S}\}^2,
  \label{eq:local_constant}
\end{equation}
where \(\mathcal N(i;S)\) contains the \(K\) nearest other observations to
\(X_{iS}\) in standardized \(S\)-coordinates, and \(\overline Y_{i,S}\) is
their mean response.

When \(n\) is small, neighboring predictors can be far enough apart for their
conditional means to differ.  This mean variation enters the
local response variance in \eqref{eq:local_constant}.  To understand this, 
write a neighboring location of a point \(z\) as \(z+u\)
and let \(v_S(z)=\Var(Y\mid X_S=z)\).  A first-order expansion gives
\[
  m_S(z+u)
  =
  m_S(z)+\nabla m_S(z)^\top u+O(\|u\|^2).
\]
Consequently, the local response variance has the leading form
\begin{equation}
  \begin{aligned}
  \Var(Y\mid X_S\text{ near }z)
  &=\E\{v_S(X_S)\mid X_S\text{ near }z\}\\
  &\quad+\nabla m_S(z)^\top
  \Sigma_S(z)\nabla m_S(z)\\
  &\quad+\text{higher-order mean terms},
  \end{aligned}
  \label{eq:local_constant_bias}
\end{equation}
where \(\Sigma_S(z)\) is the covariance of the local predictor
displacements.  For a neighborhood of radius \(h\), Lipschitz continuity of
\(v_S\) gives
\(\E\{v_S(X_S)\mid X_S\text{ near }z\}=v_S(z)+O(h)\).  The second term in
\eqref{eq:local_constant_bias} is the systematic conditional-mean drift, which is
nonnegative and inflates \(\widehat{\mathcal L}_{\mathrm{LC}}(S)\).

This bias can obscure the gain from a weak coordinate during selection and
perturb the total-effect values used for PME.  Figure~\ref{fig:local_linear_drift} illustrates the
distinction for
\(m(z+u)=u+u^2\) and a symmetric neighborhood
\(u\sim\mathrm{Unif}[-h,h]\).  The conditional-mean contribution to a
local-constant variance is \(h^2/3+4h^4/45\).  Removing the best affine
approximation leaves only \(4h^4/45\).  The figure isolates this deterministic
mean-drift component; homoscedastic conditional noise would add the same
target variance to both local approximations. This observation suggests a local linear estimation for the conditional risk, which is discussed next.

\begin{figure}[!htbp]
\centering
\includegraphics[width=\textwidth]{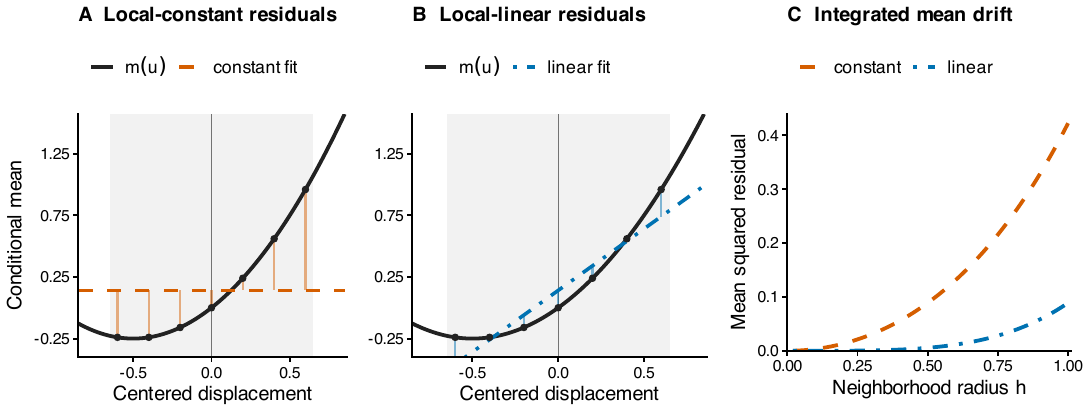}
\caption{Local-constant and local-linear residuals for
\(m(z+u)=u+u^2\) on \(u\in[-h,h]\). Panels A and B show the deterministic
residuals; Panel C gives their integrated squares. Local linearization removes
\(h^2/3\), leaving \(4h^4/45\).}
\label{fig:local_linear_drift}
\end{figure}

\subsection{Local-Linear Risk Estimation}
Let \(K_n(d)\) be the number of neighbors prescribed for a
\(d\)-dimensional local fit.  We replace the local-constant calculation in
\eqref{eq:local_constant} by a centered local-linear fit.  For a standalone
fit on a subset \(S\) of size \(d\), set \(K=K_n(d)\), center the local design
at \(X_{iS}\), and fit
\begin{equation}
  (\widehat\beta_{0,i},\widehat\beta_{1,i})
  =
  \arg\min_{\beta_0,\beta_1}
  \sum_{j\in\mathcal N(i;S)}
  \{Y_j-\beta_0-\beta_1^\top(X_{jS}-X_{iS})\}^2.
  \label{eq:local_linear}
\end{equation}
Centering makes \(\beta_0\) the fitted mean at \(X_{iS}\) and \(\beta_1\) the
local slope.  It leaves fitted values and residuals unchanged relative to the
corresponding raw-coordinate fit.  In our  implementation, we added a small ridge penalty
\(\lambda_{\mathrm{loc}}\|\beta_1\|_2^2\) to
\eqref{eq:local_linear} for numerical stability, leaving the intercept unpenalized.  Let
\begin{equation}
  g_i(S)
  =
  \widehat v^{\mathrm{LL}}_S(X_{iS})
  =
  \frac{\operatorname{SSE}_{i,S}}{K-d-1},
  \qquad
  \widehat{\mathcal L}_{\mathrm{LL}}(S)
  =
  \frac1n\sum_{i=1}^n g_i(S),
  \label{eq:local_linear_variance}
\end{equation}
where \(\operatorname{SSE}_{i,S}\) is the local residual sum of squares.
Thus \(g_i(S)\) is the pointwise local estimate of
\(v_S(X_{iS})=\Var(Y\mid X_S=X_{iS})\), and
\(\widehat{\mathcal L}_{\mathrm{LL}}(S)\) averages these estimates over the
sample.  Set \(g_i(\varnothing)=\widehat{\Var}(Y)\).
Now, using
\eqref{eq:noise_adjusted_value}, estimate the explained-variance value by
\begin{equation}
  \widehat V(S)
  =
  \left[
    \widehat{\Var}(Y)-\widehat{\mathcal L}_{\mathrm{LL}}(S)
  \right]_+,
  \qquad
  \widehat V(\varnothing)=0.
  \label{eq:estimated_explained_variance}
\end{equation}

\paragraph{Approximation and consistency.}

Fix \(S\subseteq P\), write \(Z=X_S\in\mathbb R^d\), and define
\[
  m_S(z)=\E(Y\mid Z=z),
  \qquad
  v_S(z)=\Var(Y\mid Z=z).
\]
For an interior point \(Z_i\), let
\[
  h_i
  =\max_{j\in\mathcal N(i;S)}\|Z_j-Z_i\|_2
\]
denote its projected neighborhood radius. The following result characterizes the residual bias after removing the
affine part of the conditional mean.
\begin{theorem}
\label{thm:local_linear_bias}
For the unpenalized local-linear fit, under Conditions (L1)--(L3), for an
interior point with
\(h_i\to0\), \(K\to\infty\), and fixed \(d\),
\begin{align}
&\E\{\widehat v^{\mathrm{LL}}_S(Z_i)\mid Z_1,\ldots,Z_n\}
-v_S(Z_i) \notag\\
&\quad =
O(h_i)
+O\{h_i^4\|H_{m_S}(Z_i)\|_F^2\}.
\label{eq:local_linear_bias_theory}
\end{align}
The corresponding local-constant estimator contains the additional term
\[
  O\{h_i^2\|\nabla m_S(Z_i)\|_2^2\}.
\]
Variation in \(v_S\) contributes \(O(h_i)\) to both local approximations.
The local-linear conditional-mean contribution is
\(O\{h_i^4\|H_{m_S}(Z_i)\|_F^2+h_i^6\}\), including at points where the
Hessian vanishes.
\end{theorem}

\Needspace{6\baselineskip}
The next result establishes consistency of the averaged risk estimator
for a fixed subset.
\begin{theorem}
\label{thm:subset_consistency}
Under Conditions (L1)--(L4), for fixed \(d\), if
\(K_n(d)\to\infty\) and \(K_n(d)/n\to0\), then
\begin{equation}
  \widehat{\mathcal L}_{\mathrm{LL}}(S)
  \xrightarrow{p}
  \mathcal L_Y(S).
  \label{eq:subset_consistency}
\end{equation}
\end{theorem}

%% file: sections/05_framework.tex
As in FIRST \citep{huang2025first}, selection proceeds in two steps: (i)
forward selection and (ii) backward elimination.  Efficient versions of these two steps are developed in this section.  Finally, we compute PRIME on the selected set.
\subsection{Forward Selection}

The main purpose of forward selection is to reduce the candidate set to a manageable size in
high-dimensional problems. We initialize the forward path with a global standardized
ridge fit.
Let \(Z\) denote the standardized
predictor matrix and let
\[
  \widetilde Z_j
  =
  \frac{Z_j}{\{n^{-1}\|Z_j\|_2^2\}^{1/2}}.
\]
Let \(\widehat\pi\) order the ridge coefficients by decreasing absolute value:
\begin{equation}
  \widehat b
  =
  \arg\min_b
  \|\by-\bar Y\mathbf 1-\widetilde Zb\|_2^2
  +
  \lambda_{\mathrm{init}}\|b\|_2^2,
  \qquad
  |\widehat b_{\widehat\pi_1}|
  \geq\cdots\geq
  |\widehat b_{\widehat\pi_p}|.
  \label{eq:ridge_seed}
\end{equation}
Admissible local fits require \(K_n(d)>d+1\) and
\(K_n(d)\leq\lfloor n/2\rfloor\). Define
\begin{equation}
  d_{\max,n}
  =
  \max\left\{d\leq p:
  d+1<K_n(d)\leq\lfloor n/2\rfloor\right\}.
  \label{eq:local_dimension_limit}
\end{equation}
Let \(r_p\leq d_{\max,n}\) bound the ridge path.  For
\(r=0,\ldots,r_p\), fit the global linear model using the first \(r\)
coordinates in \(\widehat\pi\), with the intercept-only model at \(r=0\), and
let \(\operatorname{RSS}(r)\) and \(d_r\) be its residual sum of squares and
slope rank.  We select the prefix size using the Bayesian information
criterion \citep{schwarz1978estimating}:
\begin{equation}
  \operatorname{BIC}(r)
  =
  n\log\{\operatorname{RSS}(r)/n\}
  +
  d_r\log n.
  \label{eq:ridge_bic}
\end{equation}
Set
\begin{equation}
  \widehat r\in\argmin_{0\leq r\leq r_p}\operatorname{BIC}(r)\;\; \textrm{and}
  \;\;
  \widehat S_{\mathrm R}
  =
  \{\widehat\pi_1,\ldots,\widehat\pi_{\widehat r}\},
  \label{eq:bic_ridge_seed}
\end{equation}
which gives the initialization for the local search discussed next. 

\paragraph{Paired risk contrasts.}

For nested sets \(S\subseteq T\subseteq P\), let \(\Gamma(S,T)\) denote the
increase in explained variance from using \(T\) rather than \(S\).
Proposition~\ref{prop:noise_decomposition} gives
\[
  \Gamma(S,T)
  =
  V(T)-V(S)
  =
  \mathcal L_M(S)-\mathcal L_M(T)
  =
  \mathcal L_Y(S)-\mathcal L_Y(T).
\]
For each comparison \(S\subseteq T\), both fits use \(K=K_n(|T|)\) neighbors,
selected separately in their respective coordinate spaces.
Define the difference at observation \(i\) by
\begin{align}
  H_i(S,T)
  &=
  g_i(S)
  -
  g_i(T), \; i=1,\ldots,n.
  \label{eq:query_risk_contrast}
  \end{align}
Its sample average estimates the gain:
  \begin{align}
  \widehat\Gamma(S,T)
  &=
  \frac1n\sum_{i=1}^n H_i(S,T),
  \label{eq:estimated_risk_contrast}
  \end{align}
For calibration, define the empirical scale of the pointwise differences
\begin{align}
  \widehat\omega^2(S,T)
  &=
  \frac{1}{n(n-1)}
  \sum_{i=1}^n
  \{H_i(S,T)-\widehat\Gamma(S,T)\}^2.
  \label{eq:query_contrast_scale}
\end{align}
For the larger-set dimension \(m=|T|\geq1\), define
\begin{equation}
  u_m=\sqrt{2\log\{\max(2,m)\}},
  \qquad
  \widehat\Gamma^{\mathrm{cal}}_m(S,T)
  =
  [\widehat\Gamma(S,T)-u_m\widehat\omega(S,T)]_+.
  \label{eq:calibrated_risk_contrast}
\end{equation}
{\widowpenalty=10000
Sampling variation can produce a positive estimated gain even when the
population gain is zero. Forward selection therefore adds a coordinate only
when its estimated gain exceeds the buffer \(u_m\widehat\omega(S,T)\).
Backward elimination discussed in Section \ref{sec:backward} then removes a coordinate when its estimated deletion loss
does not exceed the same buffer.\par}

Thus, starting from \(C_0=\widehat S_{\mathrm R}\) in (\ref{eq:bic_ridge_seed}), use
\(K=K_n(|C_t|+1)\) for both fits in every candidate comparison and define the
estimated explained-variance gain of
adding coordinate \(j\notin C_t\) by
\begin{equation}
  \widehat\Delta_j(C_t)
  =
  \widehat\Gamma(C_t,C_t\cup\{j\}).
  \label{eq:forward_gain}
\end{equation}
Rank candidates by the raw gain and add its maximizer only if the calibrated
gain is positive:
\begin{equation}
  j_{t+1}\in
  \arg\max_{j\notin C_t}\widehat\Delta_j(C_t),
  \qquad
  C_{t+1}=C_t\cup\{j_{t+1}\},
  \label{eq:forward_step}
\end{equation}
provided
\[
  \widehat\Gamma^{\mathrm{cal}}_{|C_t|+1}
  (C_t,C_t\cup\{j_{t+1}\})>0.
\]
The path continues under this condition while
\(|C_t|<d_{\max,n}\), and stops otherwise.  Denote the terminal set by
\begin{equation}
  \widehat C=C_{\widehat t}.
  \label{eq:ridge_start_screen}
\end{equation}
Backward elimination described in Section \ref{sec:backward} then removes coordinates whose deletion losses do not
exceed their buffers.  The neighborhood
rule and the numerical choices for \(r_p\) and ridge stabilization are given
in the supplement.

We next give conditions under which this ridge-start screen retains a
mean-sufficient set.  Under uniqueness, this is equivalent to covering \(A\).
Write
\[
  \Delta_j(S)=V(S\cup\{j\})-V(S).
\]
For both fixed and growing \(p\), assume that the complete ridge-ranking and
BIC rule has a stable deterministic output \(S_{\mathrm R}^\star\), which
may depend on \(n\), in the following sense:
\begin{equation}
  \Pr(\widehat S_{\mathrm R}=S_{\mathrm R}^\star)\longrightarrow1.
  \label{eq:ridge_seed_stability}
\end{equation}
If \(A\subseteq S_{\mathrm R}^\star\), this stability condition already gives
forward coverage.  Otherwise, define
\[
  \mathcal C_A
  =
  \{S_{\mathrm R}^\star\cup D:
    D\subseteq A,\ A\nsubseteq S_{\mathrm R}^\star\cup D\}.
\]

The following assumption separates the gain of the best remaining active
coordinate from those of inactive coordinates.
\begin{assumption}
\label{ass:forward_margin}
The set \(S_{\mathrm R}^\star\cup A\) has size at most
\(d_{\max,n}\).  For every \(C\in\mathcal C_A\),
\begin{align*}
  \max_{j\in A\setminus C}\Delta_j(C)&\geq\delta,\\
  \max_{j\in A\setminus C}\Delta_j(C)
  -\max_{j\in P\setminus(A\cup C)}\Delta_j(C)&\geq\eta
\end{align*}
for constants \(\delta,\eta>0\), with the second maximum interpreted as
\(-\infty\) when its index set is empty.
\end{assumption}

Let
\[
  \mathcal E_n
  =
  \{(C,j):C\in\mathcal C_A,\ j\notin C\}.
\]
Assume
\begin{equation}
  \max_{(S,j)\in\mathcal E_n}
  |\widehat\Delta_j(S)-\Delta_j(S)|
  =
  o_p(1).
  \label{eq:forward_uniform_error}
\end{equation}
Let \(M_n\geq1\) bound \(|S|+1\) over \((S,j)\in\mathcal E_n\).  Also assume
\begin{equation}
  \max_{(S,j)\in\mathcal E_n}
  \frac1n\sum_{i=1}^n
  H_i^2(S,S\cup\{j\})=O_p(1),
  \qquad
  \frac{\log\{\max(2,M_n)\}}{n}\longrightarrow0.
  \label{eq:forward_scale_condition}
\end{equation}

Now we show that the forward selection retains the active
set under these conditions.
\begin{proposition}
\label{prop:forward_coverage}
Under Assumption~\ref{ass:forward_margin},
\eqref{eq:ridge_seed_stability}--\eqref{eq:forward_scale_condition}, the
empirical forward path covers the minimal mean-sufficient set:
\begin{equation}
  \Pr(A\subseteq\widehat C)\longrightarrow1.
  \label{eq:forward_coverage}
\end{equation}
\end{proposition}

Assumption~\ref{ass:forward_margin} permits ties among active coordinates.  It
requires the best remaining active gain to be positive and separated from all
inactive gains until \(A\) is covered.  Uniform gain estimation and the
vanishing calibration buffer transfer this ordering to the empirical search.

\subsection{Backward Elimination}\label{sec:backward}

Let \(C\in\mathcal S_{\mathrm{ms}}\), equivalently \(A\subseteq C\) under the
uniqueness assumption.  The population deletion loss of coordinate \(j\in C\)
is
\begin{equation}
  D_j(C)
  =
  V(C)-V(C\setminus\{j\})
  =
  \mathcal L_M(C\setminus\{j\})-\mathcal L_M(C)
  =
  \mathcal L_Y(C\setminus\{j\})-\mathcal L_Y(C)
  =
  V(C)T_C(\{j\}).
  \label{eq:deletion_loss}
\end{equation}
It is the same one-coordinate total effect that enters the PME set function.
Deleting a coordinate with \(D_j(C)=0\) preserves mean sufficiency.

For a current set \(C\) of size \(d\), use the common neighborhood size
\(K=K_n(d)\) for the full and reduced fits, while recomputing neighborhoods in
the reduced coordinates.  The empirical deletion loss and its calibrated
version are
\begin{equation}
  \widehat D_j(C)
  =
  \widehat\Gamma(C\setminus\{j\},C),
  \qquad
  \widetilde D_j(C)
  =
  \widehat\Gamma^{\mathrm{cal}}_d(C\setminus\{j\},C).
  \label{eq:universal_deletion_loss}
\end{equation}
Here \(u_d\widehat\omega(C\setminus\{j\},C)\) is the buffer for the deletion
loss.  The multiplier follows the universal-threshold form
\citep{donoho1994ideal}, while the empirical scale adapts the buffer to the
observed variation across sample points.

Backward elimination removes one coordinate attaining
\(\min_{j\in C}\widetilde D_j(C)=0\), recomputes every full and reduced fit on
the smaller set, and stops when all calibrated deletion losses are positive.
Denote the terminal set by \(\widehat A\).  Sequential recomputation matters
because both the projected neighborhoods and the remaining deletion losses
change after a coordinate is removed. The following proposition gives conditions under which backward elimination
preserves the active coordinates retained by forward selection.
\begin{proposition}
\label{prop:active_retention}
Suppose \(\Pr(A\subseteq\widehat C)\to1\), and let \(M_n\) bound
the dimensions visited by backward elimination. In the following conditions,
\(S\) ranges over the visited sets satisfying \(A\subseteq S\), and
\(j\) ranges over \(S\). Uniformly over these sets, assume that
\[
  \mathcal L_Y(S\setminus\{j\})-\mathcal L_Y(S)\geq\Delta
\]
for some \(\Delta>0\) and every \(j\in A\), and that the same-\(K\) deletion
losses are uniformly consistent.  If
\begin{equation}
  \max_{S,j}
  \frac1n\sum_{i=1}^n
  H_i^2(S\setminus\{j\},S)=O_p(1),
  \qquad
  \frac{\log\{\max(2,M_n)\}}{n}\longrightarrow0,
  \label{eq:vanishing_query_threshold}
\end{equation}
then
\[
  \Pr(A\subseteq\widehat A)\longrightarrow1.
\]
\end{proposition}

Condition~\eqref{eq:vanishing_query_threshold} makes the buffers vanish
uniformly, so uniformly consistent estimates of the positive active
deletion losses exceed them with probability tending to one.
Thus backward elimination preserves the coordinates needed for PRIME;
by Proposition~\ref{prop:selected_set_equivalence}, surplus coordinates
do not affect the population allocation.

\input{sections/after_selection}
\FloatBarrier

\subsection{PRIME Allocation on the Selected Coordinates}

Now that we have identified the active set $ \widehat A$, we can compute the PME using only $ \widehat A$, which gives PRIME.  Define
\[
  \widehat{\mathcal L}_{\widehat A}
  =\widehat{\mathcal L}_{\mathrm{LL}}(\widehat A),
  \qquad
  \widehat V_{\widehat A}
  =
  \left[\widehat{\Var}(Y)-\widehat{\mathcal L}_{\widehat A}\right]_+.
\]
On the coverage event \(A\subseteq\widehat A\),
Proposition~\ref{prop:noise_decomposition} gives
\(\mathcal L_Y(\widehat A)=\sigma_\varepsilon^2\), so we use
\(\widehat\sigma_\varepsilon^2=\widehat{\mathcal L}_{\widehat A}\) as the
selected-set noise-floor estimate in \eqref{eq:estimated_signal_risk}.  If
coverage fails, this risk can also contain omitted-signal variation.  For
\(U\subseteq\widehat A\), this gives
the noise-adjusted risk
\begin{equation}
  \widehat{\mathcal R}_{\widehat A}(U)
  =
  \left[
    \widehat{\mathcal L}_{\mathrm{LL}}(U)
    -\widehat{\mathcal L}_{\widehat A}
  \right]_+,
  \qquad
  \widehat{\mathcal R}_{\widehat A}(\varnothing)
  =
  \widehat V_{\widehat A}.
  \label{eq:noise_adjusted_subset_risk}
\end{equation}
When \(\widehat V_{\widehat A}>0\), estimate the selected-set total effect by
\begin{equation}
  \widehat T_{\widehat A}(B)
  =
  \frac{\widehat{\mathcal R}_{\widehat A}(\widehat A\setminus B)}
       {\widehat V_{\widehat A}},
  \label{eq:estimated_total_effect}
\end{equation}
with the values at \(B=\varnothing\) and \(B=\widehat A\) fixed at zero and
one.  Although, the positive part forces estimates below the selected-set noise floor to zero, the resulting total-effect estimates can still fail to increase with
the deleted set \(B\), as the population values do.  We enforce this ordering
by projecting onto the closed convex set
\begin{equation}
  \mathcal M_{\widehat A}
  =
  \left\{
  t:t(\varnothing)=0,\;t(\widehat A)=1,\;
  t(S)\leq t(U)\text{ whenever }S\subseteq U
  \right\}.
  \label{eq:normalized_monotone_class}
\end{equation}
Let \(\Pi_C\) denote Euclidean projection onto \(\mathcal M_C\), and set
\begin{equation}
  \widetilde T_{\widehat A}
  =
  \Pi_{\widehat A}(\widehat T_{\widehat A})
  =
  \arg\min_{t\in\mathcal M_{\widehat A}}
  \sum_{B\subseteq\widehat A}
  \{t(B)-\widehat T_{\widehat A}(B)\}^2.
  \label{eq:monotone_projection}
\end{equation}
The reduced PRIME estimate is
\begin{equation}
  \widehat\psi_{\widehat A}
  =
  \PV^0(\widehat A,\widetilde T_{\widehat A})
  \in\mathbb R^{|\widehat A|},
  \label{eq:estimated_selected_pme}
\end{equation}
and the zero-padded PRIME vector is
\begin{equation}
  \widehat\psi_j
  =
  \begin{cases}
    (\widehat\psi_{\widehat A})_j, & j\in\widehat A,\\
    0, & j\notin\widehat A.
  \end{cases}
  \label{eq:prime_estimator}
\end{equation}
The complete PRIME procedure is detailed in Algorithm~\ref{alg:framework}.

In the following results, \(\widehat{\mathcal L}_Y(U)\) denotes the
allocation-stage local-linear risk estimate, with
\(\widehat{\mathcal L}_Y(\varnothing)=\widehat{\Var}(Y)\).
For fixed nonempty \(C\), construct \(\widehat T_C\) by
\eqref{eq:noise_adjusted_subset_risk}--\eqref{eq:estimated_total_effect},
with \(C\) in place of \(\widehat A\). Proposition \ref{prop:pme_stability} shows how consistency of the subset-risk estimates
carries over to the PME allocation on a fixed coordinate set.
\begin{proposition}
\label{prop:pme_stability}
For fixed \(C\) with \(V(C)>0\), suppose that \(T_C\) has a greatest null
coalition.  If
\[
  \max_{S\subseteq C}
  |\widehat{\mathcal L}_Y(S)-\mathcal L_Y(S)|
  \xrightarrow{p}0,
\]
then \(\Pr(\widehat V_C>0)\to1\).  Conditional on \(\widehat V_C>0\),
\[
  \|\Pi_C(\widehat T_C)-T_C\|_2\xrightarrow{p}0
\]
and
\[
  \PV^0\{C,\Pi_C(\widehat T_C)\}
  \xrightarrow{p}
  \PV^0(C,T_C).
\]
\end{proposition}

\begin{algorithm}[!b]
\footnotesize
\caption{Proportional inclusive marginal effects (PRIME)}
\label{alg:framework}
\begin{algorithmic}[1]
\Require Data \((\bX,\by)\); \(K_n(\cdot)\);
\(r_p\); \((\lambda_{\mathrm{init}},\lambda_{\mathrm{loc}})\)
\Ensure Reduced coordinates \(\widehat A\) and importance vector
\(\widehat\psi\) when allocation is defined
\Statex \textbf{Ridge initialization}
\State Standardize \(\bX\)
\State Compute \(\widehat b\) and \(\widehat\pi\) as in
\eqref{eq:ridge_seed}
\State Compute \(\widehat r\) and \(\widehat S_{\mathrm R}\) as in
\eqref{eq:ridge_bic}--\eqref{eq:bic_ridge_seed}
\State Compute \(d_{\max,n}\) as in \eqref{eq:local_dimension_limit}
\State \(C\gets\widehat S_{\mathrm R}\)
\Statex \textbf{Forward selection}
\While{\(C\neq P\) and \(|C|<d_{\max,n}\)}
  \State \(K\gets K_n(|C|+1)\)
  \State Evaluate \(\widehat\Gamma(C,C\cup\{j\})\) and
  \(\widehat\Gamma^{\mathrm{cal}}_{|C|+1}(C,C\cup\{j\})\),
  \(j\in P\setminus C\), as in
  \eqref{eq:local_linear_variance} and
  \eqref{eq:query_risk_contrast}--\eqref{eq:calibrated_risk_contrast}
  \State Choose \(j^\star\in
  \arg\max_{j\notin C}\widehat\Gamma(C,C\cup\{j\})\)
  \If{\(\widehat\Gamma^{\mathrm{cal}}_{|C|+1}
  (C,C\cup\{j^\star\})=0\)}
    \State \textbf{break}
  \EndIf
  \State \(C\gets C\cup\{j^\star\}\)
\EndWhile
\Statex \textbf{Backward elimination}
\While{\(|C|>0\)}
  \State \(K\gets K_n(|C|)\)
  \State Evaluate \(\widetilde D_j(C)\), \(j\in C\), as in
  \eqref{eq:universal_deletion_loss}
  \If{\(\min_{j\in C}\widetilde D_j(C)>0\)}
    \State \textbf{break}
  \EndIf
  \State Choose \(j^\star\in\arg\min_{j\in C}\widetilde D_j(C)\)
  \State \(C\gets C\setminus\{j^\star\}\)
\EndWhile
\Statex \textbf{PRIME allocation}
\State \(\widehat A\gets C\)
\If{\(\widehat A=\varnothing\)}
  \State Report unavailable allocation and \textbf{stop}
\EndIf
\State Compute the retained explained variance \(\widehat V_{\widehat A}\)
\If{\(\widehat V_{\widehat A}\leq0\)}
  \State Report unavailable allocation and \textbf{stop}
\EndIf
\State Form \(\widehat T_{\widehat A}(B)\),
\(B\subseteq\widehat A\), as in
\eqref{eq:noise_adjusted_subset_risk}--\eqref{eq:estimated_total_effect}
\State \(\widetilde T_{\widehat A}\gets
\Pi_{\widehat A}(\widehat T_{\widehat A})\) according to
\eqref{eq:monotone_projection}
\State Set \(\widehat\psi\) as in
\eqref{eq:estimated_selected_pme}--\eqref{eq:prime_estimator}
\State \Return \((\widehat A,\widehat\psi)\)
\end{algorithmic}
\end{algorithm}


The following theorem combines the selection and allocation results to
establish consistency of PRIME.
\begin{theorem}
\label{thm:end_to_end}
Let \(p=p_n\) be allowed to grow while the marginal law of the
square-integrable response \(Y\) remains fixed, and write \(A_n\) for the
corresponding unique minimal mean-sufficient set.  Suppose that \(|A_n|\)
and the terminal selected dimension are bounded by a fixed integer \(q\),
\(\inf_n V(A_n)>0\), and the conditions of
Propositions~\ref{prop:forward_coverage} and \ref{prop:active_retention}
hold with margins independent of \(n\).  Assume explicitly that the risks
are uniformly consistent over all subsets of the data-dependent selected set:
\begin{equation}
  \max_{U\subseteq\widehat A}
  |\widehat{\mathcal L}_Y(U)-\mathcal L_Y(U)|
  \xrightarrow{p}0.
  \label{eq:selected_risk_uniform_consistency}
\end{equation}
Let \(\mathcal G_n=\{\widehat A\neq\varnothing,
\widehat V_{\widehat A}>0\}\) be the event that the allocation is available.
Then \(\Pr(\mathcal G_n)\to1\), and, for every \(\epsilon>0\),
\begin{equation}
  \Pr\!\left\{
    \max_{j\in P}|\widehat\psi_j-\psi_j^\star|>\epsilon
    \,\middle|\,\mathcal G_n
  \right\}\longrightarrow0.
  \label{eq:end_to_end}
\end{equation}
\end{theorem}


%% file: sections/after_selection.tex
For illustration, consider the Friedman function \citep{friedman1991multivariate}
\[
 Y=10\sin(\pi X_{a_1}X_{a_2})+20(X_{a_3}-1/2)^2
   +10X_{a_4}+5X_{a_5}+\epsilon,
 \qquad \epsilon\sim N(0,1),
\]
with noise independent of the five active predictors and \(p=100\).
Let \(X_j=\Phi(Z_j)\), where \(Z\sim N(0,\Sigma)\). We use three different simulation designs.
(i) The independent design uses \(\Sigma=I\) and \(a=\{1,\ldots,5\}\).
(ii) The spaced AR(1) design uses \(\Sigma_{jk}=0.5^{|j-k|}\) and
\(a=\{10,30,50,70,90\}\). Adjacent latent coordinates are correlated,
while consecutive active coordinates are separated by 20 positions.
(iii) The proxy design uses \(a=\{1,\ldots,5\}\),
\(\Sigma_{j,5+j}=\Sigma_{5+j,j}=0.5\) for \(j\leq5\),
unit diagonal, and all other off-diagonal entries zero. The covariance
parameters describe the latent Gaussian variables before applying the $\Phi$-transformation.
In addition to the proposed screening method (PRIME screen), we include a few other variable selection methods FIRST \citep{huang2025first}, LASSO
\citep{tibshirani1996regression}, FOCI \citep{azadkia2021simple}, and FORD
\citep{azadkia2025integrated} in the comparison. We use the exact active-set recovery as the performance measure, which is defined as the proportion of replications for which
\(\widehat A=A\). Figure \ref{fig:friedman_selection_illustration} shows the  exact active-set recovery of the five methods for a range of values of $n$ from 100 to 2000 each with 50 replications. We can see that the PRIME screen attains the highest observed exact-recovery rate in all three simulation designs.

\begin{figure}[!htbp]
\centering
\includegraphics[width=\textwidth]{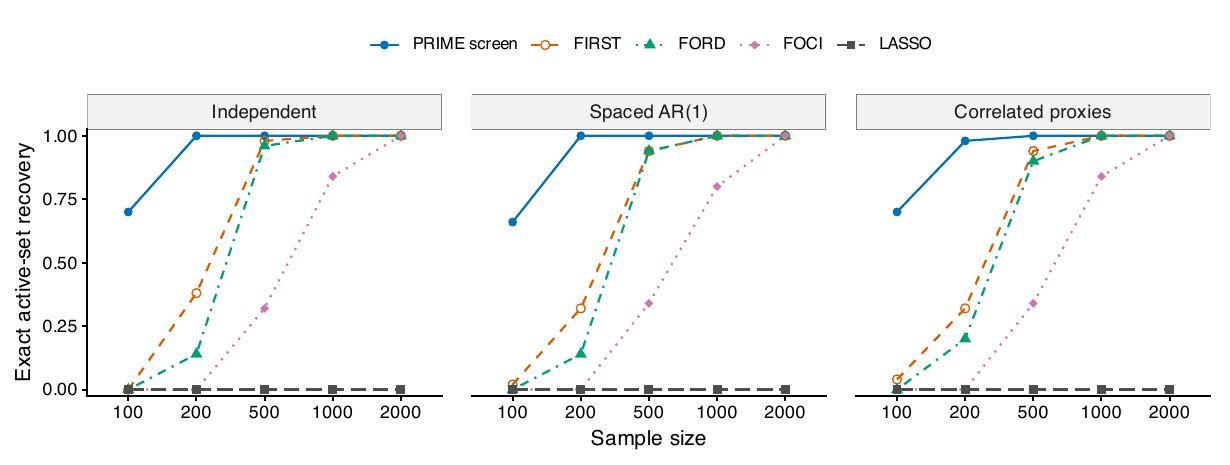}
\caption{Exact active-set recovery for Friedman function at \(p=100\) based on 50 replications.}
\label{fig:friedman_selection_illustration}
\end{figure}

%% file: sections/06_experiments.tex
In this section, we compare PRIME to the other competing methods on a test function and a real dataset. Comparisons with additional test functions and datasets are provided in the supplementary file.

\input{sections/new_friedman}

\ifConcreteMain
  \input{sections/new_concrete}
\else
  \input{sections/new_auto_mpg}
\fi

%% file: sections/new_friedman.tex
\subsection{A Test Function}
Consider again the Friedman function \citep{friedman1991multivariate}:
\[
Y=10\sin(\pi X_1X_2)+20(X_3-1/2)^2+10X_4+5X_5+\epsilon,
\qquad \epsilon\overset{iid}{\sim} N(0,1).
\]
The inputs are generated using
\[
X_j=\Phi(Z_j),\qquad
X_{5+j}=\Phi\!\left(\rho Z_j+\sqrt{1-\rho^2}E_j\right),
\qquad j=1,\ldots,5,
\]
where $Z_j, E_j \overset{iid}{\sim} N(0,1)$. Thus, $X_1,\ldots,X_5$ are active and $X_6,\ldots,X_{10}$ are paired proxies.


We compare PRIME to PME. Although \cite{herin2024proportional} discusses only deterministic outputs, a simple adjustment for noisy outputs is provided in the R package \texttt{sensitivity} \citep{iooss2022sensitivity} in which the total Sobol' index is adjusted as
\[\widehat T_P^{\mathrm{kNN}}(B)=\dfrac{\widehat{\mathcal L}_{\mathrm{kNN}}(P\setminus B)
       -\widehat{\mathcal L}_{\mathrm{kNN}}(P)}
      {\widehat{\Var}(Y)},\]
where \(\widehat{\mathcal L}_{\mathrm{kNN}}(S)\) denotes the
average sample response variance over neighborhoods in \(S\)-coordinates with six neighbors and \(\widehat{\mathcal L}_{\mathrm{kNN}}(\varnothing)=\widehat{\Var}(Y)\). 
For a baseline reference, we also compare PRIME to the Shapley effects after screening given by the FIRST-SHAP method \citep{huang2025first}. SHAP \citep{lundberg2017unified} is a widely used Shapley-based technique for interpreting complex black-box models, but the SHAP values computed for the entire dataset can be aggregated to produce a global importance measure. Here we compute the mean absolute SHAP values as in \cite{williamson2020spvim} by fitting gradient-boosted trees to each dataset and averaging absolute tree-path-dependent SHAP values over the same observations. Tree depth and number of trees are selected by three-fold cross-validation using the one-standard-error rule. The results with $n=100$ and $\rho=0.9$ with 30 replications are shown in Figure \ref{fig:friedman-methods}.

\begin{figure}[H]
\centering
\includegraphics[width=0.85\textwidth]{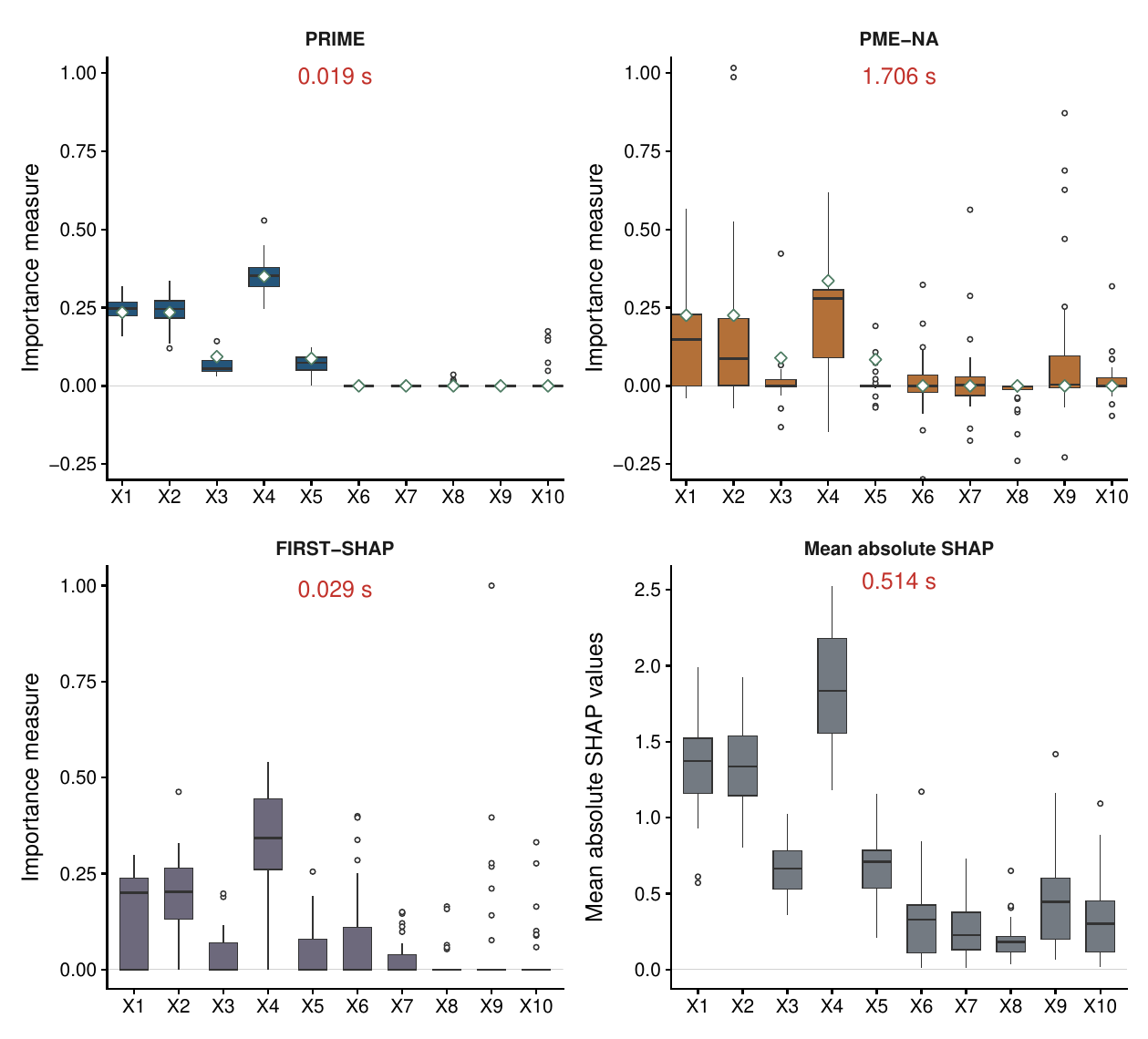}
\caption{Variable importance for Friedman function. Mean absolute SHAP is in response units. Diamonds mark the method-specific population PME targets: unit-signal PME for PRIME and response-variance PME for PME-NA. Computational times shown in red are median elapsed seconds per replication.
Thirteen estimates of PME-NA lie outside the displayed range.}
\label{fig:friedman-methods}
\end{figure}

We can see that PRIME gives the best performance. It identifies $X_1,\ldots,X_5$ as the important variables  and assigns zero importance to the inactive variables $X_6,\ldots,X_{10}$ most of the time. On the other hand, the noise-adjusted PME (denoted as PME-NA) shows large variation in the importance measures of each variable and even include negative values. The FIRST-SHAP also suffers from large variation and many times misidentifies $X_6$ and $X_7$ as important variables. The disadvantage of using Shapley-based effects can be clearly seen in SHAP which allocates nonzero importance measures to the inactive variables $X_6,\ldots,X_{10}$. 


The computational times of the four methods are also shown in Figure \ref{fig:friedman-methods}. We can see that PRIME is the fastest of the four. PME-NA evaluates $2^{10}=1,024$ subsets for allocation, whereas PRIME evaluates $2^{|\widehat A|}$ subsets after screening, with a median of 32 subsets in this comparison. The computational time for mean absolute SHAP also includes the model fitting time of gradient-boosted trees.

We repeated the PRIME Vs PME comparison for all combinations of $n\in\{100,200,500\}$ and $\rho\in\{0,0.5,0.9\}$. Table \ref{tab:pme_accuracy} shows the mean relative half-$L_1$ PME error of PRIME and PME-NA. PRIME has lower mean relative PME error than PME-NA in each of the nine conditions. To check the effectiveness of variable selection, we also included PRIME with all the 10 variables (PRIME-Full). We can see that its error is higher than that of PRIME showing that variable selection indeed helps in improving the estimates of the importance measure.

\input{tables/pme_accuracy}

%% file: tables/pme_accuracy.tex
\begin{table}[H]
\centering
\caption{Mean relative half-\(L_1\) PME error over 30 datasets per condition. }
\begin{tabular}{rrccc}
\toprule
\(n\) & \(\rho\) & PRIME & PRIME-Full & PME-NA \\
\midrule
100 & 0 & 0.101 & 0.140 & 0.644 \\
100 & 0.5 & 0.107 & 0.143 & 0.404 \\
100 & 0.9 & 0.103 & 0.134 & 0.748 \\
\addlinespace
200 & 0 & 0.067 & 0.104 & 0.428 \\
200 & 0.5 & 0.062 & 0.094 & 0.398 \\
200 & 0.9 & 0.083 & 0.109 & 0.773 \\
\addlinespace
500 & 0 & 0.047 & 0.080 & 0.741 \\
500 & 0.5 & 0.047 & 0.072 & 0.294 \\
500 & 0.9 & 0.045 & 0.058 & 0.843 \\
\bottomrule
\end{tabular}
\label{tab:pme_accuracy}
\end{table}

%% file: sections/new_concrete.tex
\subsection{ A Real Dataset}
The UCI Concrete data \citep{yeh2007concretedata} contains 1,030 strength measurements and eight 
inputs  pertaining to concrete ingredients and age, and one response: compressive strength. In our analysis, we replace water by the
water-to-binder ratio given by
\(\mathrm{water}/(\mathrm{cement}+\mathrm{slag}+\mathrm{ash})\), which has an established engineering relevance to strength \citep{yeh1998concrete}.

We took 100 repeated samples
of 824 observations without replacement and estimated importance measures for
PRIME, PME-NA, FIRST-SHAP, and mean absolute SHAP following the same protocol mentioned in the previous section.
The results are shown in Figure \ref{fig:concrete-p8-methods}. Their computational times are also shown in the same figure.

We can see that PRIME identifies the water-to-binder ratio and age as the two most important variables. Cement is also given some importance which shows that the ratio is not enough to capture the effect of both cement and water. On the other hand, PME-NA identifies age as the most important variable and then cement. The water-to-binder ratio is given low importance most of the time, which contradicts the engineering knowledge \citep{yeh1998concrete}. FIRST-SHAP did a good job in this example, although it identifies fly ash to be more important than cement, which is suspicious. Mean absolute SHAP agrees with PRIME on the three most important variables, but gives nonzero important measures to all other inputs which may not be realistic.

\begin{figure}[H]
\centering
\includegraphics[width=0.85\textwidth]{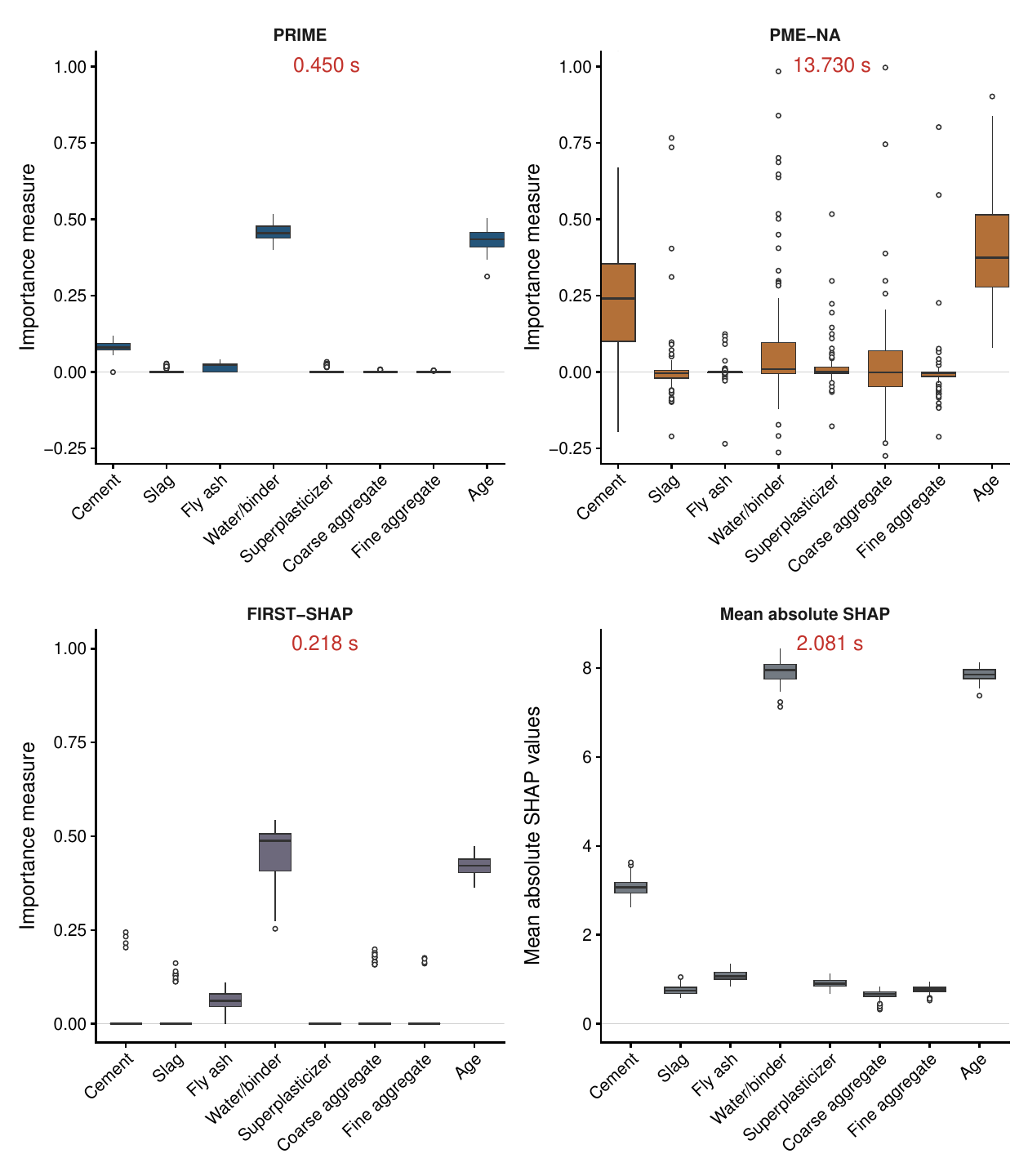}
\caption{Variable importance for Concrete. Mean absolute SHAP is in response units. Median computation times are shown in red. Twenty-nine estimates of PME-NA lie outside the displayed range.}
\label{fig:concrete-p8-methods}
\end{figure}


%% file: sections/07_conclusion.tex
Shapley effect is thought be the go-to measure for variable importance, but it has the drawback that inactive variables that are correlated with active variables can get a substantial share of the importance. Proportional marginal effects (PME) overcome this issue, which is in the same spirit of using PMVD instead of the LMG measure in linear regression.  However, PME has two limitations: it is developed for deterministic models and has high computational cost. This work is an attempt to overcome these two limitations. Under an additive error assumption, we have developed estimators that will remove the effect of noise in the estimation of PME. Moreover, we proactively use a novel model-free variable selection technique to identify the active set and include only them in the computation of PME, thereby drastically reducing the computational cost. The new PME measure is named as PRIME. An R package \texttt{prime} and Python codes  are available at \url{https://github.com/ChoYoungSeo/prime-importance-paper}.

Another major contribution of our work is in the estimation. Most of the previous works used nearest neighbor-based techniques for estimation of global sensitivity measures. We have replaced it with a local-linear regression-based estimate. We found this to be much more accurate, especially when the sample size is low. Our simulation studies show that PRIME is able to provide a reliable and meaningful measure for the importance of variables compared to the other methods in the literature.

This article deals only with continuous outputs and therefore cannot be used in classification problems. Extending PRIME to binary and categorical outputs is an important direction for future research.

 \vspace{.25in}
 \noindent  {\Large\bf Acknowledgments}

 \noindent This research is supported by  a U.S. National Science Foundation grant DMS-2310637.

 \vspace{.25in}
 \noindent  {\Large\bf \hyperlink{prime.supplement.start}{Supplementary Materials}}

 \noindent The file includes proofs, experimental protocols, and additional simulations.
 



%% file: sections/supplement_proofs.tex
We use the neighbor-count convention of Section~\ref{sec:first}.

\subsection{Proof of Proposition~\ref{prop:selected_set_equivalence}}

\begin{proof}
Fix \(C\in\mathcal S_{\mathrm{ms}}\).  Since \(P\) is finite, \(C\) contains
an inclusion-minimal mean-sufficient subset.  Uniqueness makes that subset
\(A\), so \(A\subseteq C\).  If \(B\subseteq C\setminus A\), then
\(A\subseteq C\setminus B\).  Mean sufficiency and monotonicity give
\[
  V(C\setminus B)=V(C)=V(A),
\]
and hence \(T_C(B)=0\).

Conversely, suppose \(T_C(B)=0\).  Then
\[
  V(C\setminus B)=V(C)=V(A)=V(P).
\]
Equation~\eqref{eq:mean_sufficiency_identity_theory} shows that
\(C\setminus B\) is
mean sufficient.  It contains an inclusion-minimal mean-sufficient subset.
Uniqueness makes that subset \(A\), so \(A\subseteq C\setminus B\) and
\(B\subseteq C\setminus A\).  Therefore
\[
  E=C\setminus A
\]
is the greatest null coalition of \(T_C\): it contains every subset with
zero total effect.

For every \(D\subseteq A\),
\begin{align*}
  T_C(E\cup D)
  &=
  \frac{V(C)-V\{C\setminus(E\cup D)\}}{V(C)}\\
  &=
  \frac{V(A)-V(A\setminus D)}{V(A)}
  =
  T_A(D).
\end{align*}
Theorem~1 of \citet{herin2024proportional} gives the canonical extended
proportional value.  Here the greatest null coalition \(E\) is also the
unique null coalition of maximum cardinality.  The formula assigns zero to every
coordinate in \(E\) and evaluates the remaining coordinates through the
shifted set function \(D\mapsto T_C(E\cup D)\) on \(C\setminus E=A\).  The
preceding identity shows that this set function is exactly \(T_A\), proving
the equality between the allocations on \(C\) and \(A\).  Taking \(C=P\)
identifies their common padded allocation with \(\psi^\star\), proving
\eqref{eq:mean_sufficient_invariance}.
\end{proof}

\paragraph{Nonunique minimal representations.}
When the family of mean-sufficient sets has several inclusion-minimal elements,
the full-coordinate PME is still defined through the null coalitions of maximum
cardinality in the canonical extension \citep{herin2024proportional}.  A selected minimal
representative need not preserve that full-coordinate allocation.  For example, two
almost-surely identical coordinates can each form a sufficient singleton; PME
on either singleton assigns all mass to the selected representative, whereas
the symmetric game on all coordinates divides the mass between them.  The selected-set
equivalence and subsequent selection theory therefore assume a unique minimal
set.

\subsection{Proof of Proposition~\ref{prop:noise_decomposition}}

\begin{proof}
For any \(S\subseteq P\),
\[
  Y-m_S(X_S)
  =
  M-m_S(X_S)+\varepsilon.
\]
The cross term vanishes because
\[
  \E\left[\{M-m_S(X_S)\}\varepsilon\right]
  =
  \E\left[
    \{M-m_S(X_S)\}\E(\varepsilon\mid X)
  \right]
  =0.
\]
Hence
\[
  \mathcal L_Y(S)
  =
  \E\{Y-m_S(X_S)\}^2
  =
  \mathcal L_M(S)+\sigma_\varepsilon^2,
\]
which proves \eqref{eq:noise_decomposition}.
Moreover,
\(
\Var(Y)=\Var(M)+\sigma_\varepsilon^2
\)
under \(\E(\varepsilon\mid X)=0\).  Equation~\eqref{eq:variance_decomposition}
gives \(V(S)=\Var(M)-\mathcal L_M(S)\); substituting the two preceding
identities proves \eqref{eq:noisy_value_identification}.
\end{proof}

\subsection{Conditions for the Local-Linear Results}

Fix a subset \(S\), write \(Z=X_S\in\mathbb R^d\), and let \(h_i\) be the
radius of the projected \(K_n(d)\)-nearest-neighbor set at point \(Z_i\).  If
\(U_i\) contains the centered neighbor displacements as rows, set
\[
  M_{K_n(d)}=I_{K_n(d)}-K_n(d)^{-1}\mathbf 1\mathbf 1^\top,
  \qquad
  G_i=U_i^\top M_{K_n(d)}U_i.
\]
The local-linear results use the following conditions.

\medskip
\noindent\textbf{(L1) Projected design.}
The support of \(Z\) is compact with Lipschitz boundary, and its density is
bounded above and away from zero on the support.

\smallskip
\noindent\textbf{(L2) Regression functions and moments.}
The function \(m_S\) has bounded continuous third derivatives,
\(v_S(z)=\Var(Y\mid Z=z)\) is Lipschitz and bounded above and away from
zero, and the conditional fourth moment of \(Y-m_S(Z)\) is uniformly bounded.

\smallskip
\noindent\textbf{(L3) Local design.}
Uniformly over interior points, the eigenvalues of
\(G_i/\{K_n(d)h_i^2\}\) are bounded above and away from zero.

\smallskip
\noindent\textbf{(L4) Neighborhood overlap.}
No observation belongs to more than \(B_n(S)\) projected neighborhoods, where
\(B_n(S)/n\to0\).

Conditions (L1)--(L3) are standard fixed-dimensional random-design conditions
for local polynomial regression \citep{fan1996local,ruppert2003semiparametric}.
Condition (L4) controls dependence among pointwise residual variances induced
by overlapping neighborhoods.

\subsection{Proof of Theorem~\ref{thm:local_linear_bias}}

\begin{proof}
Fix an interior point \(Z_i=z_i\), write
\[
  u_{ij}=Z_j-z_i,
\]
let \(U_i\) be the \(K\)-by-\(d\) matrix with rows \(u_{ij}^{\top}\), and set
\[
  M_K=I_K-K^{-1}\mathbf 1\mathbf 1^\top,
  \qquad
  G_i=U_i^\top M_KU_i.
\]
The local-linear hat matrix and residual operator are
\[
  \mathbf H_i
  =K^{-1}\mathbf 1\mathbf 1^\top
   +M_KU_iG_i^{-1}U_i^\top M_K,
  \qquad
  L_i=I_K-\mathbf H_i.
\]
Condition (L3) makes \(G_i\) nonsingular.  Thus \(\mathbf H_i\)
is the orthogonal projector onto the local affine design, so
\[
  L_i\mathbf 1=0,
  \qquad
  L_iU_i=0,
  \qquad
  L_i^2=L_i,
  \qquad
  \operatorname{tr}(L_i)=K-d-1.
\]
Taylor expansion gives
\[
  m_S(z_i+u_{ij})
  =
  m_S(z_i)
  +
  \nabla m_S(z_i)^\top u_{ij}
  +
  \frac12u_{ij}^\top H_{m_S}(z_i)u_{ij}
  +
  O(h_i^3).
\]
Let \(\xi_{S,j}=Y_j-m_S(Z_j)\), and let
\(\boldsymbol\xi_{S,i}\) collect these conditional regression residuals over
the \(K\) neighbors of observation \(i\).  Conditional on the predictor locations,
the local residual is
\[
  L_i(Y_j)_{j\in\mathcal N(i;S)}
  =
  L_i\{m_S(Z_j)\}_{j\in\mathcal N(i;S)}
  +
  L_i\boldsymbol\xi_{S,i}.
\]

The constant and affine terms are annihilated.  Since
\(\|L_i\|_{\mathrm{op}}=1\), the quadratic Taylor term contributes
\[
  O\{
    K h_i^4
    \|H_{m_S}(z_i)\|_F^2
  \}.
\]
The third-order remainder contributes \(O(Kh_i^6)\).  The cross product
between the quadratic and third-order terms is bounded by twice their
geometric mean.  Dividing by \(K-d-1\), using fixed \(d\) and \(K\to\infty\),
gives \(O\{h_i^4\|H_{m_S}(z_i)\|_F^2+h_i^6\}\).
The \(O(h_i^6)\) remainder remains in the conditional-mean contribution,
including when the Hessian vanishes; it is absorbed by the \(O(h_i)\) term
in the stated overall bound.

For the stochastic component, write
\(V_i=\operatorname{diag}\{v_S(Z_j):j\in\mathcal N(i;S)\}\).  Then
\[
  \E\{
    \boldsymbol\xi_{S,i}^\top
    L_i
    \boldsymbol\xi_{S,i}
    \mid Z_1,\ldots,Z_n
  \}
  =\operatorname{tr}(L_iV_i).
\]
Write \(V_i=v_S(Z_i)I_K+\Delta_i\).  Lipschitz continuity gives
\(\|\Delta_i\|_{\mathrm{op}}=O(h_i)\).  Because \(L_i\) is a positive
semidefinite projector,
\[
  |\operatorname{tr}(L_i\Delta_i)|
  \leq
  \|\Delta_i\|_{\mathrm{op}}\operatorname{tr}(L_i)
  =O\{h_i(K-d-1)\}.
\]
After division by \(K-d-1\), the stochastic expectation therefore equals
\(v_S(Z_i)+O(h_i)\).  Combining these bounds proves
\eqref{eq:local_linear_bias_theory}.

For a local-constant fit, the affine term is not projected out.  Its squared
local average is
\[
  O\{
    h_i^2\|\nabla m_S(z_i)\|_2^2
  \}.
\]
\end{proof}

\paragraph{Numerical slope nugget.}
The implementation adds the slope-only nugget
\(\lambda_{\mathrm{loc}}\) when solving the local normal equations.  Let
\(\mu_{\min,i}\) be the smallest eigenvalue of \(G_i\) and define
\[
  \omega_{\mathrm{loc},i}
  =\frac{\lambda_{\mathrm{loc}}}
  {\mu_{\min,i}+\lambda_{\mathrm{loc}}}.
\]
Replacing \(G_i^{-1}\) by
\((G_i+\lambda_{\mathrm{loc}}I_d)^{-1}\) leaves the intercept unpenalized and
gives
\[
  L_{\mathrm{loc},i}U_i\nabla m_S(z_i)
  =M_KU_i\lambda_{\mathrm{loc}}
  (G_i+\lambda_{\mathrm{loc}}I_d)^{-1}\nabla m_S(z_i),
\]
while
\[
  \operatorname{tr}(L_{\mathrm{loc},i}^2)
  =K-d-1+
  \sum_{\ell=1}^d
  \left\{\frac{\lambda_{\mathrm{loc}}}
  {\mu_{i\ell}+\lambda_{\mathrm{loc}}}\right\}^2.
\]
The same calculation therefore adds
\[
  O\{h_i^2\|\nabla m_S(Z_i)\|_2^2\omega_{\mathrm{loc},i}^2\}
  +O(d/K)
\]
to the bound in Theorem~\ref{thm:local_linear_bias}.  These terms vanish in
the fixed-subset consistency argument whenever \(h_i\to0\) and \(d/K\to0\).
For a fixed positive nugget, \(\omega_{\mathrm{loc},i}\) need not tend to zero:
the relevant comparison is with \(\mu_{\min,i}\), which has order
\(Kh_i^2\) under (L3).  Thus the ridge fit can retain an affine residual even
as the neighborhoods shrink.  The preceding bound and the proof below give
fixed-subset consistency for this fit, while the affine-removal bound in
Theorem~\ref{thm:local_linear_bias} concerns the unpenalized projection.

\subsection{Proof of Theorem~\ref{thm:subset_consistency}}

\begin{proof}
Write \(K=K_n(d)\), \(q=K-d-1\), and
\(\mathcal Z_n=(Z_1,\ldots,Z_n)\).  The argument also covers the numerical
slope nugget.  Let \(R_i\) denote either \(L_i\) or
\(L_{\mathrm{loc},i}\).  Both are symmetric contractions satisfying
\(R_i\mathbf 1=0\), and
\[
  \operatorname{tr}(R_i^2)=q+\delta_i,
  \qquad 0\leq\delta_i\leq d,
\]
with \(\delta_i=0\) for the unpenalized fit.  The projected density in (L1)
ensures full affine rank almost surely, including for boundary neighborhoods.

Compactness and (L2) make \(m_S\) bounded and uniformly continuous.  Let
\(a(t)\) be its modulus of continuity on the support, so \(a(t)\to0\) as
\(t\downarrow0\).  If \(m_i\) collects the neighbor conditional means, then
\[
  \|R_im_i\|_2^2
  =\|R_i\{m_i-m_S(Z_i)\mathbf 1\}\|_2^2
  \leq K a(h_i)^2.
\]
Writing \(V_i=v_S(Z_i)I_K+\Delta_i\) as above gives
\(\|\Delta_i\|_{\mathrm{op}}\leq C h_i\).  Hence, for either residual fit,
\[
  \left|\E\{g_i(S)\mid\mathcal Z_n\}-v_S(Z_i)\right|
  \leq\frac Kq\{a(h_i)^2+C h_i\}+\frac{C d}{q}.
\]
This bound applies at the boundary as well as in the interior.  Under (L1)
and \(K/n\to0\), the nearest-neighbor radius at a fixed sample point
converges to zero in probability: every fixed-radius ball about an interior
support point has positive probability, so its neighbor count eventually
exceeds \(K\) with probability tending to one.  Boundary points have
probability zero.  Empirical coordinate standardization preserves this
shrinkage because the sample coordinate variances converge to positive
limits under (L1).  Since the support is compact, the preceding bound is
uniformly bounded.  Dominated convergence and exchangeability therefore give
\[
  \frac1n\sum_{i=1}^n
  \left|\E\{g_i(S)\mid\mathcal Z_n\}-v_S(Z_i)\right|
  \xrightarrow{p}0.
\]

For the stochastic terms, put \(A_i=R_i^2/q\) and
\(T_i=g_i(S)-\E\{g_i(S)\mid\mathcal Z_n\}\).  Conditional independence of
the observation-level residuals gives
\[
  T_i=\boldsymbol\xi_{S,i}^{\top}A_i\boldsymbol\xi_{S,i}
      -\operatorname{tr}(A_iV_i)
      +2m_i^{\top}A_i\boldsymbol\xi_{S,i}.
\]
The conditional fourth-moment bound and Cauchy--Schwarz bound the variance
of this quadratic-plus-linear expression by
\[
  \Var(T_i\mid\mathcal Z_n)
  \leq C\{\|A_i\|_F^2+\|A_im_i\|_2^2\}
  \leq \frac{C K}{q^2}=O(K^{-1}),
\]
uniformly over all sample points.  Here \(\|R_i\|_{\mathrm{op}}\leq1\) and
\(\|m_i\|_2^2\leq C K\).  Cauchy--Schwarz controls the dependence from
overlapping neighborhoods directly:
\[
  \Var\left(\frac1n\sum_{i=1}^nT_i\,\middle|\,\mathcal Z_n\right)
  \leq
  \left\{\frac1n\sum_{i=1}^n
  \sqrt{\Var(T_i\mid\mathcal Z_n)}\right\}^2
  =O(K^{-1})=o(1).
\]
Combining these bounds with the law of large numbers for \(v_S(Z_i)\) yields
\[
  \widehat{\mathcal L}_{\mathrm{LL}}(S)
  \xrightarrow{p}
  \E\{v_S(Z)\}=\mathcal L_Y(S).
\]
\end{proof}

\subsection{Proof of Proposition~\ref{prop:forward_coverage}}

\begin{proof}
If \(A\subseteq S_{\mathrm R}^\star\),
\eqref{eq:ridge_seed_stability} proves the result directly.  Consider the
remaining case.
Let
\[
  \epsilon_n
  =
  \max_{(S,j)\in\mathcal E_n}
  |\widehat\Delta_j(S)-\Delta_j(S)|.
\]
By \eqref{eq:query_contrast_scale},
\[
  \widehat\omega^2(S,S\cup\{j\})
  \leq
  \frac{1}{n-1}
  \left\{
    \frac1n\sum_{i=1}^n
    H_i^2(S,S\cup\{j\})
  \right\}.
\]
Equation~\eqref{eq:forward_scale_condition} therefore gives
\[
  \max_{(S,j)\in\mathcal E_n}
  u_{|S|+1}
  \widehat\omega(S,S\cup\{j\})=o_p(1).
\]
On the event
\[
  2\epsilon_n<\eta
  \quad\text{and}\quad
  \epsilon_n+
  \max_{(S,j)\in\mathcal E_n}
  u_{|S|+1}
  \widehat\omega(S,S\cup\{j\})<\delta,
\]
intersect with \(\widehat S_{\mathrm R}=S_{\mathrm R}^\star\).  Both events
have probability tending to one.  Suppose the current empirical set
\(C\in\mathcal C_A\) does not yet contain \(A\), and let \(j_A\) maximize
\(\Delta_j(C)\) over \(j\in A\setminus C\).  For every
\(k\in P\setminus(A\cup C)\), Assumption~\ref{ass:forward_margin} gives
\begin{align*}
  \widehat\Delta_{j_A}(C)
  &\geq \Delta_{j_A}(C)-\epsilon_n\\
  &\geq \Delta_k(C)+\eta-\epsilon_n\\
  &>\widehat\Delta_k(C).
\end{align*}
Thus every empirical maximizer lies in \(A\setminus C\).  Its raw gain is at
least \(\delta-\epsilon_n\), so its calibrated gain is positive on the same
event.  The next set again belongs to \(\mathcal C_A\), unless it already
contains \(A\).  Induction adds one previously unselected active coordinate at
each step.  Since \(|S_{\mathrm R}^\star\cup A|\leq d_{\max,n}\), the path
covers \(A\) before reaching the dimension limit.
\end{proof}

\subsection{Proof of Proposition~\ref{prop:active_retention}}

\begin{proof}
By \eqref{eq:query_contrast_scale},
\[
  \widehat\omega^2(S\setminus\{j\},S)
  \leq
  \frac{1}{n-1}
  \left\{
    \frac1n\sum_{i=1}^n
    H_i^2(S\setminus\{j\},S)
  \right\}.
\]
Since \(u_{|S|}\leq\sqrt{2\log\{\max(2,M_n)\}}\),
\eqref{eq:vanishing_query_threshold} gives
\[
  \max_{S,j}
  u_{|S|}
  \widehat\omega(S\setminus\{j\},S)=o_p(1).
\]
For \(j\in A\), the deletion gap is at least \(\Delta\).  Uniform consistency
gives, uniformly over the visited active-containing sets,
\[
  \widehat D_j(S)\geq\Delta+o_p(1).
\]
Equation~\eqref{eq:vanishing_query_threshold} therefore yields
\[
  \widetilde D_j(S)
  \geq
  \Delta+o_p(1)>0
\]
uniformly for active coordinates.  The backward rule removes only a
coordinate with zero calibrated loss, so no active coordinate is removed with
probability tending to one.  Induction over the sequentially recomputed path
proves the result.
\end{proof}

\subsection{Continuity at Games with a Greatest Null Coalition}

For a finite nonempty player set \(D\), let \(\mathcal M_D\) be the
normalized monotone class in \eqref{eq:normalized_monotone_class}.
The norm \(\|s-t\|_\infty\) is the maximum absolute difference over
subsets of \(D\).  We first establish the boundary property used in both
allocation results.

\begin{lemma}[Local and uniform continuity of the canonical extension]
\label{lem:greatest_null_continuity}
If \(t\in\mathcal M_D\) has a greatest null coalition \(E\), then
\(s\mapsto\PV^0(D,s)\) is continuous at \(t\) relative to
\(\mathcal M_D\).
For \(0<c\leq1\), let \(\mathcal K_{D,c}\) consist of the games in
\(\mathcal M_D\) for which some \(E\subsetneq D\) satisfies
\[
  t(B)=0\quad(B\subseteq E),
  \qquad
  t(B)\geq c\quad(B\nsubseteq E).
\]
For every fixed integer \(q\), there is a nondecreasing function
\(\omega_{q,c}(r)\to0\) as \(r\downarrow0\) such that, whenever
\(1\leq|D|\leq q\), \(t\in\mathcal K_{D,c}\), and
\(s\in\mathcal M_D\),
\[
  \|\PV^0(D,s)-\PV^0(D,t)\|_\infty
  \leq\omega_{q,c}(\|s-t\|_\infty).
\]
\end{lemma}

\begin{proof}
Write \(d=|D|\) and \(A=D\setminus E\).  Since \(t(D)=1\),
\(A\neq\varnothing\).  The greatest-null property and finiteness imply
\[
  c_t:=\min_{B\nsubseteq E}t(B)>0.
\]
If \(E=\varnothing\), all nonempty values of \(t\) are positive, and
continuity follows directly from the finite products and normalization in
\eqref{eq:pv_order_weight}.  Suppose \(E\neq\varnothing\), and first let
\(s\in\mathcal M_D\) be positive on every nonempty subset, with
\(r=\|s-t\|_\infty<c_t/2\).  Then
\[
  s(B)\leq r\quad(B\subseteq E),
  \qquad
  c_t/2\leq s(B)\leq1\quad(B\nsubseteq E).
\]
For a permutation \(\pi\), denote its first \(k\) players by
\(C_k(\pi)\) and its unnormalized weight by
\[
  W_s(\pi)=\prod_{k=1}^d s\{C_k(\pi)\}^{-1}.
\]
Call an ordering good if all players in \(E\) precede all players in
\(A\).  For a bad ordering \(\pi\), form the good ordering
\(\pi^\circ\) by moving its \(E\) players first while preserving the
relative order within each of \(E\) and \(A\).  In
\[
  \frac{W_s(\pi)}{W_s(\pi^\circ)}
  =\prod_{k=1}^d
    \frac{s\{C_k(\pi^\circ)\}}{s\{C_k(\pi)\}},
\]
the factors before the first \(A\) player in \(\pi\) equal one.
That first player occurs by position \(|E|\), so its factor is at most
\(2r/c_t\).  Every subsequent denominator corresponds to a set containing
an \(A\) player and is at least \(c_t/2\), while its numerator is at most
one.  Therefore
\[
  W_s(\pi)\leq r(2/c_t)^d W_s(\pi^\circ).
\]
Each good ordering has at most \(d!\) preimages under this map.  Hence the
probability \(\alpha_s\) of a bad ordering under the normalized weights
\(p_s\) satisfies
\begin{equation}
  \alpha_s\leq d!r(2/c_t)^d.
  \label{eq:bad_order_weight_bound}
\end{equation}

Conditional on a good ordering, the weight factors into the product for the
ordering within \(E\) and the product for the ordering within \(A\).
Summing out the former gives precisely the proportional-order weights for
the positive normalized game
\[
  h_s(\varnothing)=0,
  \qquad h_s(B)=s(E\cup B),\quad\varnothing\neq B\subseteq A.
\]
The total marginal contribution of the \(E\) players in a good ordering
is \(s(E)\leq r\).  For an \(A\) player, its marginal contribution agrees
with that in \(h_s\), except that the first \(A\) player has the additional
subtraction \(s(E)\).  Thus the conditional expected allocation differs
coordinatewise by at most \(r\) from \(\PV(A,h_s)\), padded with zeros
on \(E\).  Every marginal contribution is in \([0,1]\), so including
the bad orderings increases this bound by at most \(\alpha_s\).  It follows
that
\[
  \|\PV(D,s)-\operatorname{pad}_E\PV(A,h_s)\|_\infty
  \leq r+\alpha_s\longrightarrow0.
\]
As \(s\to t\), the positive game \(h_s\) converges to \(t_E\).
Its nonempty values are bounded below by \(c_t/2\), so ordinary proportional
values converge by \eqref{eq:pv_order_weight}.  The canonical formula
\eqref{eq:continuous_pv_unique_null} identifies
\(\operatorname{pad}_E\PV(A,t_E)\) with \(\PV^0(D,t)\).
This proves convergence for positive perturbations.

For an arbitrary sequence \(s_n\in\mathcal M_D\) converging to \(t\),
use the defining equal-zero regularization of \(s_n\).  For each \(n\),
choose a sufficiently small positive regularization \(s_n^+\) so that
\[
  \|s_n^+-s_n\|_\infty<1/n,
  \qquad
  \|\PV(D,s_n^+)-\PV^0(D,s_n)\|_\infty<1/n.
\]
Such a choice exists by the canonical-extension definition; the endpoint
value one remains unchanged.  Since \(s_n^+\to t\), the positive-case
result and the triangle inequality prove local continuity for all
nonnegative monotone perturbations.

For uniformity, relabel \(D\) as \(\{1,\ldots,d\}\), which preserves
the allocation rule and the norms.  For a fixed \(d\) and \(E\), the
conditions defining \(\mathcal K_{D,c}\) are closed conditions in
\([0,1]^{2^d}\).  Their finite union over \(E\subsetneq D\) is therefore
compact.  If the stated uniform conclusion failed, there would be
\(d_n\leq q\), \(t_n\in\mathcal K_{D_n,c}\), and
\(s_n\in\mathcal M_{D_n}\) with
\(\|s_n-t_n\|_\infty\to0\), but allocation differences bounded away from
zero.  Pass to a subsequence with a fixed dimension and a fixed greatest
null set.  Compactness gives a further subsequence \(t_n\to t\) retaining
the same null set and non-null lower bound \(c\); also \(s_n\to t\).
Local continuity at \(t\) makes both allocations converge to
\(\PV^0(D,t)\), a contradiction.  Taking the supremum of the allocation
difference over games within distance \(r\), and the finitely many
dimensions \(d\leq q\), defines the required modulus
\(\omega_{q,c}\).
\end{proof}

The canonical extension is not continuous on the entire nonnegative domain.
For a two-player normalized game with singleton values \(a,b>0\),
\(\PV=(a/(a+b),b/(a+b))\).  At \(a=b=0\), the canonical equal-zero
regularization gives \((1/2,1/2)\), while the path
\((a,b)=(\epsilon,\epsilon^2)\) converges to \((1,0)\).
This distinguishes the equal-zero limiting construction in
\citet[Theorem~1 and Appendix~B]{herin2024proportional} from continuity
under arbitrary perturbations.  The latter is supplied here by
Lemma~\ref{lem:greatest_null_continuity} at games with a greatest null coalition.

\subsection{Proof of Proposition~\ref{prop:pme_stability}}

\begin{proof}
Write
\(e_n=\max_{U\subseteq C}|\widehat{\mathcal L}_Y(U)-\mathcal L_Y(U)|\).
Since the empty-subset risk is the response variance, the positive-part
map is 1-Lipschitz, and \(V(C)>0\),
\[
  |\widehat V_C-V(C)|\leq2e_n,
  \qquad \Pr\{\widehat V_C>0\}\longrightarrow1.
\]
For a proper subset \(B\subset C\), let
\(r_C(B)=\mathcal L_Y(C\setminus B)-\mathcal L_Y(C)\).
Then \(0\leq r_C(B)\leq V(C)\), and its positive-part estimate differs
from it by at most \(2e_n\).  On \(e_n<V(C)/4\), division by
\(\widehat V_C\geq V(C)/2\) gives
\[
  \|\widehat T_C-T_C\|_\infty\leq\frac{8e_n}{V(C)}.
\]
The two endpoints are exact.  Fixed dimension therefore gives
\(\|\widehat T_C-T_C\|_2\to_p0\).
Since \(T_C\in\mathcal M_C\) and Euclidean projection onto a closed
convex set is nonexpansive,
\[
  \|\Pi_C(\widehat T_C)-T_C\|_2
  \leq\|\widehat T_C-T_C\|_2\xrightarrow{p}0.
\]
Apply Lemma~\ref{lem:greatest_null_continuity} at the fixed game \(T_C\)
to obtain allocation convergence.  All these conclusions hold on the
available-allocation events, whose probability tends to one.
\end{proof}

\subsection{Proof of Theorem~\ref{thm:end_to_end}}

\begin{proof}
Propositions~\ref{prop:forward_coverage} and \ref{prop:active_retention}
give \(\Pr(\mathcal B_n)\to1\), where
\(\mathcal B_n=\{A_n\subseteq\widehat A\}\).
Let \(v_0=\inf_n V(A_n)>0\) and \(v_Y=\Var(Y)<\infty\).
The fixed marginal law of \(Y\) supplies a common upper bound:
\(V(A_n)\leq v_Y\) for every \(n\).
On \(\mathcal B_n\), write \(C=\widehat A\).  Proposition
\ref{prop:selected_set_equivalence} gives \(V(C)=V(A_n)\), greatest
null coalition \(E=C\setminus A_n\), and equality of the zero-padded
population allocation with \(\psi^\star\).
For any \(B\nsubseteq E\), choose \(j\in B\cap A_n\).  Monotonicity
and the active-deletion margin in Proposition~\ref{prop:active_retention}
give
\begin{equation}
  T_C(B)\geq T_C(\{j\})
  =\frac{\mathcal L_Y(C\setminus\{j\})-\mathcal L_Y(C)}{V(C)}
  \geq\frac{\Delta}{v_Y}=:c>0.
  \label{eq:selected_game_separation}
\end{equation}
Thus \(T_C\in\mathcal K_{C,c}\) uniformly over the covered selected
sets.  In particular, \(T_{A_n}(D)=T_C(E\cup D)\geq c\) for every
nonempty \(D\subseteq A_n\); a separate shifted-game separation
assumption is redundant.

Set
\[
  e_n=\max_{U\subseteq\widehat A}
  |\widehat{\mathcal L}_Y(U)-\mathcal L_Y(U)|=o_p(1).
\]
On \(\mathcal B_n\cap\{e_n<v_0/4\}\), the same positive-part and
ratio bounds as in the preceding proof give
\[
  \widehat V_C\geq v_0/2,
  \qquad
  \|\widehat T_C-T_C\|_\infty\leq8e_n/v_0.
\]
Since \(V(A_n)>0\), \(A_n\neq\varnothing\).  This event therefore
implies \(C\neq\varnothing\) and \(\mathcal G_n\), proving
\(\Pr(\mathcal G_n)\to1\).  For \(|C|\leq q\), nonexpansiveness and
the finite-dimensional norm inequalities yield
\[
  \|\Pi_C(\widehat T_C)-T_C\|_\infty
  \leq\|\Pi_C(\widehat T_C)-T_C\|_2
  \leq\|\widehat T_C-T_C\|_2
  \leq 2^{q/2}\frac{8e_n}{v_0}.
\]
The uniform part of Lemma~\ref{lem:greatest_null_continuity}, with the
common \(q\) and \(c\) in \eqref{eq:selected_game_separation}, now gives
\[
  \|\PV^0\{C,\Pi_C(\widehat T_C)\}-\PV^0(C,T_C)\|_\infty
  \leq\omega_{q,c}(2^{q/2}8e_n/v_0)\xrightarrow{p}0
\]
on events whose probability tends to one.  On \(\mathcal B_n\),
zero padding preserves this maximum error because the full-coordinate population target \(\psi^\star\) is
zero outside \(C\).  Hence, for every \(\epsilon>0\),
\[
  \Pr\!\left(\mathcal G_n\cap
    \left\{\max_{j\in P}|\widehat\psi_j-\psi_j^\star|>\epsilon\right\}
  \right)\longrightarrow0.
\]
Divide by \(\Pr(\mathcal G_n)\to1\) to obtain
\eqref{eq:end_to_end}.
\end{proof}

%% file: sections/supplement_experiments.tex
\subsection{Common implementation}

The implementation requires \(n\geq10\) and \(p\geq1\).
For the primary comparisons, each predictor is centered and divided by its
sample standard deviation within the current dataset. PRIME, PME-kNN, and
FIRST-SHAP receive the identical standardized matrix, with further internal
rescaling disabled. SHAP uses the same observations without standardization.
The local neighborhood size is
\[
  K_n(d)=\left\lceil\max\{1.5\sqrt n,3d\}\right\rceil,
\]
subject to \(K_n(d)\leq\lfloor n/2\rfloor\).  The numerical regularizers are
\(\lambda_{\mathrm{loc}}=\lambda_{\mathrm{init}}=10^{-6}\), applied to the
local slopes and global ridge fit, respectively.  The local-linear theory uses
the unpenalized affine projection.
The global ridge ordering defines a nested path through
\[
  r_p=\min\{\lceil1.5\log\{\max(2,p)\}\rceil,d_{\max,n}\}
\]
coordinates.  BIC selects the seed size along this path, including the
intercept-only choice \(r=0\).  The local path starts from the resulting ridge
set and continues while its leading calibrated gain is positive.
Backward elimination uses
\[
  \widehat\Gamma^{\mathrm{cal}}_{|C|}
  (C\setminus\{j\},C)
\]
with the common neighbor count \(K=K_n(|C|)\) for both fits.  After every
removal, it recomputes the projected neighborhoods and paired pointwise
deletion losses on the remaining set.
For the stated neighborhood rule,
\(d_{\max,n}=\min\{p,\lfloor n/6\rfloor\}\), which ensures
\(3d\leq\lfloor n/2\rfloor\).
The default calculation uses every observation as a point at which to fit a local model.  Exact
projected nearest neighbors are used throughout.

All methods receive the same generated data within a replication.  A task
seed is a deterministic function of the response mechanism, dependence
design, sample size, ambient dimension, and replication. The additional
experiments run datasets in parallel with one computational thread per task.
Timings were
recorded under R 4.6.1 on the aarch64 macOS platform with native and BLAS
thread counts fixed at one.  Raw selected sets, coordinate allocations,
subset-value counts, timing components, and errors are stored before summaries
are formed.

The displayed selector-to-PME comparison evaluates exact proportional
allocation when the selected set is nonempty, has at most 12 coordinates,
and has retained explained variance exceeding
\(\sqrt{\epsilon_{\mathrm{mach}}}\,\widehat{\Var}(Y)\);
remaining runs are recorded as incomplete.  After monotone
projection, values at most
\(\sqrt{\epsilon_{\mathrm{mach}}}\) are treated as numerical zeros in
the extended proportional-value recursion.
In the implementation used for these experiments, final effects with absolute value at most
\(\sqrt{\epsilon_{\mathrm{mach}}}\) are also treated as numerical zeros. These fixed numerical
tolerances approximate the exact operations used in the consistency statements.
Half-\(L_1\) loss,
\[
  \frac12\sum_{j=1}^p|\widehat\psi_j-\psi_j^\star|,
\]
is the primary allocation endpoint. For nonnegative normalized allocations,
it lies in \([0,1]\) and equals their total variation distance. The normalized estimator ablations use this loss; native FIRST-SHAP distributions are shown
coordinatewise without assigning an ambient-Shapley estimation error.
Loss summaries omit incomplete allocations; completion counts retain every
attempt.  Reported subset-value counts include fixed endpoint values.

\subsection{Friedman and primary observed-data comparisons}
The main Friedman comparison uses five independent uniform active inputs,
five paired Gaussian-copula proxies, and independent $N(0,1)$ noise.
There are 30 datasets for each of the nine $(n,\rho)$ conditions. The unit-signal
regression-mean PME reference is evaluated by deterministic quadrature of the active-input
functional ANOVA followed by the proportional-value recursion; it is zero
on all proxies. The reference used by these new comparisons is separate from
the Monte Carlo reference files of the earlier additional experiments.

PRIME retains both statistical buffers and its neighborhood and ridge settings.
PME-kNN uses the unmodified \texttt{pme\_knn} implementation from
\texttt{sensitivity} 1.27.1, the version specified in the authors' public code.
We use six neighbors with
\texttt{rescale=FALSE}, \texttt{marg=TRUE}, \texttt{tol=NULL}, and
\texttt{n.limit=2001}. Table~\ref{tab:pme_accuracy} uses
PME-kNN with \texttt{noise=TRUE} (PME-NA).
Figures~\ref{fig:friedman-methods} and~\ref{fig:concrete-p8-methods}
compare PRIME, PME-NA, FIRST-SHAP, and mean absolute SHAP. The supplementary comparisons also
include \texttt{noise=FALSE}. The PRIME, PME-kNN, and FIRST-SHAP procedures receive these standardized predictors
and return the allocations used in the comparisons. External paired
replications provide the sampling design; each call computes one point estimate.
Mean absolute SHAP uses XGBoost 2.1.4 and SHAP 0.46.0. The candidate models
have 200, 500, or 1,000 trees with maximum depth 2 or 4. We fix the learning
rate at 0.05, the squared-error objective, \texttt{reg\_lambda=1},
\texttt{min\_child\_weight=1}, and \texttt{tree\_method=hist}, with no row
or column subsampling. Three-fold cross-validation uses mean validation
\(R^2\) and the one-standard-error rule, preferring fewer trees and then
shallower depth. The selected model is refitted to all observations, and
exact tree-path-dependent SHAP values are averaged in absolute value over
the same observations.

For the Friedman design, let \(T_P\) be the normalized regression-mean
total-effect function in \eqref{eq:selected_total_effect} on the full
coordinate set. For nonempty \(B\subseteq P\), the two noise conventions have
population games
\[
 T_P^{\mathrm{FALSE}}(B)=(1-\gamma)+\gamma T_P(B),
 \qquad T_P^{\mathrm{TRUE}}(B)=\gamma T_P(B),
\]
with both games zero at the empty set. The FALSE convention therefore retains
a positive noise contribution and has positive population proxy allocations.
Diamonds in Figure~\ref{fig:friedman-methods} mark unit-signal PME for PRIME
and the proportional value of \(T_P^{\mathrm{TRUE}}\) for PME-NA.
For Table~\ref{tab:pme_accuracy}, let \(\psi\) be unit-signal population
PME and let \(c_m\psi\) be method \(m\)'s native population target. We evaluate
\[
 E_m=\frac{1}{2c_m}\sum_{j=1}^{10}
       |\widehat\phi_{m,j}-c_m\psi_j|.
\]
Here \(c_m=1\) for PRIME and PRIME-Full, and
\(c_m=\gamma=\Var\{m(X)\}/\Var(Y)=0.9597204\) for PME-kNN
with \texttt{noise=TRUE}.
The population scaling in the loss leaves the native estimates unchanged;
it retains their total-allocation error. Signed estimates can give \(E_m>1\).
The controlled estimator ablations below use the common
noise-adjusted projected game. The timing labels for PRIME, PME-NA, and FIRST-SHAP in
Figure~\ref{fig:friedman-methods} summarize a separate serial run over the same
30 datasets, with one native thread, rotating method order, and input loading
outside the timed interval. The 120 timed calls cover PRIME, FIRST-SHAP,
and both PME-kNN settings; one warm-up per procedure is excluded.
Each of these intervals includes standardization and the complete method call. Friedman SHAP timings are measured separately and include model selection, fitting, and explanation. All four methods were timed serially on the same 100 Concrete samples using a single thread.
FIRST-SHAP uses its native selection step, two neighbors, two forward passes,
and local fits at all observations, with internal rescaling disabled in both
stages. It computes Shapley effects on its retained
set and pads the remaining coordinates by zero. Signed estimates are retained.
All 810 evaluations of the three complete procedures (PRIME, PME-kNN
with \texttt{noise=TRUE}, and FIRST-SHAP) across the nine Friedman conditions
are available. The 30 additional \texttt{noise=FALSE} fits use the
\(n=100,\rho=0.9\) datasets displayed in Figure~\ref{fig:friedman-methods}. The allocation-rule comparison applies PME and
Shapley to identical full-variable monotone-projected local-linear subset
estimates, changing only the allocation rule.
The Full LC-\(K_n\) ablation uses the same ten-coordinate ground set and, for
each nonempty subset \(U\), the same \(K_n(|U|)\) nearest other observations as
the local-linear estimator.  It replaces the affine fit by the local sample
variance with divisor \(K_n(|U|)-1\), uses the full-set risk as the noise floor,
and applies the same positive-part normalization and monotone projection.
Thus Full LL-PME and Full LC-\(K_n\) differ only in local fit order. Both are
evaluated on all 270 Friedman datasets.
Full-variable Shapley population references are calculated by deterministic
Gaussian quadrature and independently checked by conditional integration.
They use the same normalized signal-variance target as the PME reference.

\ifConcreteMain
  \input{sections/supplement_concrete_case}
\else
  \input{sections/supplement_auto_mpg_case}
\fi

\subsection{Population references}

For the additional Branin, Friedman, OTL, piston, and borehole experiments, population closed
Sobol' effects are approximated by conditional Gaussian-copula pick--freeze samples
of size 100,000.  Ackley--3 uses 500,000 reference draws.  The reference
files store every subset value, the random seed, and half-sample stability
checks.  Bratley references are
obtained from exact moments of its nested products, and the quadratic-function reference is
analytic.  Each closed-effect set function is converted to total-effect form
before applying the extended proportional value.

In the paired-proxy designs, each proxy is correlated with one active input
but is absent from the regression function.  The active-input joint law is
unchanged as the proxy correlation varies.  The reference PME therefore
equals the active-only allocation padded by zeros, as established by
Proposition~\ref{prop:selected_set_equivalence}.

The subset-value study uses the quadratic response
\[
  f_{\mathrm{quad}}(x)=\tfrac12(x_1+x_2)^2+x_3^2.
\]
For independent standard uniform inputs, its nonzero ANOVA variance
components are \(61/720\), \(61/720\), \(64/720\), and \(5/720\) for
\(X_1\), \(X_2\), \(X_3\), and the \(X_1\)--\(X_2\) interaction,
respectively.  The resulting
population PME is
\((0.33246,0.33246,0.33508)\).

The coordinate-recovery and selector-to-PME studies use the three-variable Ackley function
\citep{ackley1987connectionist}, restricted to
\([-5,5]^3\),
\[
  f_{\mathrm A}(x)
  =-20\exp\left\{-0.2
    \left(\frac13\sum_{j=1}^3z_j^2\right)^{1/2}\right\}
   -\exp\left\{\frac13\sum_{j=1}^3\cos(2\pi z_j)\right\}
   +20+e,
  \qquad z_j=10x_j-5.
\]
Under the independent active-input law used for the allocation references,
permutation symmetry and efficiency give the exact population PME
\[
  (1/3,\;1/3,\;1/3).
\]
The maximum coordinatewise difference between the two half-reference
allocations is 0.00277.

The initialization study uses the same \(n=p=100\) benchmark designs with 10
replications per predictor law.  It compares the fixed ridge seed, the
BIC-gated ridge path, and an empty start.  A separate suppressor design uses
50 replications at each displayed sample size and noise level.

Conditional on the generated inputs, independent Gaussian errors are added
to every response.  In the
screening and selector-link studies, their standard deviations are 2 for
Branin, 1 for Friedman, and 10\% of the sample signal standard deviation for
Ackley--3.  The low-dimensional and correlated-proxy studies use 10\% of
the sample signal standard deviation.  In the ambient-scaling study, Branin and
Friedman retain their absolute standard deviations and the remaining
responses use the 10\% rule.

\subsection{Simulation designs}

The synthetic predictors have uniform \([0,1]\) margins.  Write
\(X_j=\Phi(Z_j)\), where \(Z\sim\mathcal N_p(0,\Sigma)\) and \(\Phi\) is the
standard normal distribution function.  All simulation correlation parameters
refer to \(Z\).  Independent inputs use \(\Sigma=I_p\).  Spaced AR(1) uses
\(\Sigma_{jk}=0.5^{|j-k|}\); for a response with \(s\) active inputs, their
positions are \(\operatorname{round}\{(j-1/2)p/s\}\), \(j=1,\ldots,s\).
At \(p=100\), these are \(\{25,75\}\) for Branin, \(\{17,50,83\}\) for
Ackley--3, and \(\{10,30,50,70,90\}\) for Friedman.  Independent and
paired-proxy designs place the active inputs first.  In the paired-proxy
design, \(Z_j\) and \(Z_{s+j}\) have correlation \(\rho\),
\(j=1,\ldots,s\), and all other off-diagonal covariances are zero.
The screening and selector-to-PME proxy designs use
\(\rho=0.5\).

The conditional-risk comparison uses Branin, the quadratic response, Friedman, Ishigami, OTL,
piston, and borehole responses at \(n\in\{100,300\}\), under independent and
AR(1) inputs with latent correlation 0.5 and the first \(s\) inputs active.
Each cell contains 100 paired
replications.  Both estimators use the known active set \(A\), for which
\(\mathcal L_Y(A)\) equals the known additive-noise variance.  In each
replication, the estimation error is divided by the sample variance of
\(Y\); the reported normalized RMSE is the root mean square of these
normalized errors across replications.  This diagnostic evaluates risk
estimation on \(A\), rather than across all of its subsets.
Coordinate recovery uses Branin, Ackley--3, and Friedman in
\(p=100\) dimensions, five sample sizes, three predictor laws, and 50 paired
replications per cell. At \(n=p=100\), the same three responses under the
independent and paired-proxy laws are used to compare the PME error induced
by five screening methods over 30 paired replications.

The estimator comparison uses Branin and Bratley in dimensions eight and ten.
The sample sizes are 100, 250, and 500, and the proxy correlations are 0,
0.5, and 0.9, with 30 paired replications per cell.  The correlated-proxy
comparison embeds the same responses in \(p=20\) dimensions at
\(n\in\{100,200,500\}\) and \(\rho=0.9\), with 20 replications.  Ambient
dimension is varied from 20 to 500 at \(n=500\) for Branin, Friedman, OTL,
piston, and borehole, with independent uniform inputs, the active inputs
first, and 30 replications per cell.

\subsection{Additional observed-data analyses}

The Concrete Compressive Strength data contain 1,030 complete observations
on cement, blast-furnace slag, fly ash, water, superplasticizer, coarse
aggregate, fine aggregate, curing age, and compressive strength
\citep{yeh2007concretedata}.  The engineering representation replaces water by
the water-to-binder ratio while retaining the other seven predictors.
Airfoil Self-Noise contains 1,503 observations on five aerodynamic and
acoustic predictors \citep{brooks1989airfoil}.  Auto MPG contains 392 complete
observations after removing six records with missing horsepower; region of
origin is represented by two indicators, giving eight predictors
\citep{quinlan1993autompg}.

PME intervals use 200 paired nonparametric bootstrap samples.  The primary
representation of each dataset is analyzed by PRIME and full-coordinate
LL-PME.  Prediction uses k-nearest-neighbor regression evaluated by five
repeats of five-fold nested cross-validation.  Screening is fit on each
outer training fold, and the neighbor count is selected separately for each
coordinate set by five-fold validation within that training fold from
\(\{5,10,20,30,40\}\).

\subsection{Comparisons}

FIRST uses total-index selection \citep{huang2025first}.  LASSO \citep{tibshirani1996regression} uses Gaussian
linear regression with ten-fold cross-validation and the one-standard-error
penalty choice. FOCI and FORD use their
native stopping rules and scale-standardized predictors
\citep{azadkia2021simple,azadkia2025integrated}.  NNVS applies the greedy
NCMD screening algorithm with ten nearest neighbors, stopping when the best
available score decreases \citep{chatterjee2026ncmd}.

To compare screening methods, each selected set from local-linear screening,
FIRST, LASSO, FOCI, and FORD is passed to the same local-linear total-effect
estimator and PME operator.  The low-dimensional experiment instead compares
subset-value estimators: PRIME, a full-coordinate
local-linear estimator, the controlled 2NN-PME estimator, and an NCMD plug-in
estimator. In these historical comparisons, 2NN-PME standardizes predictors
and estimates conditional risk from the sample response variance within each
two-point neighborhood comprising the evaluation point and its nearest other observation
\citep{broto2020variance}. It subtracts the full retained-set risk, normalizes
by estimated signal variance, and applies our monotone projection and PME
allocator. NCMD-PME uses five
nearest neighbors: their averaged response products estimate the second
moment of the conditional mean, from which the cross-pair estimate of
\(\{\E(Y)\}^2\) is subtracted before normalization by the sample response
variance \citep{chatterjee2026ncmd}.  Oracle-coordinate local-linear PME measures the
remaining estimation error when the active coordinates are known.  Direct
2NN-PME evaluates all \(2^p\) subset values, comprising the empty subset and
\(2^p-1\) nonempty subsets. At \(p=20\), these counts are
1,048,576 and 1,048,575. The correlated-proxy and ambient-dimension
experiments compare complete screen--reduce--allocate methods and record
screening and allocation times separately.  Exact PME is evaluated for
retained sets of at most 12 coordinates, or 4,096 total (4,095 nonempty)
coalition values;
larger selected sets are recorded as computationally incomplete.
The selector-to-PME figure displays four selectors; LASSO is also evaluated
and completes 74 of its 180 allocations.

\subsection{Acceleration using representative observations}

The main procedure uses all \(n\) observations in the outer average.  Data
Twinning provides a representative-observation approximation.  For the standardized
row \(W_i=(X_i,Y_i)\), let
\(\mathcal Q\subseteq\{1,\ldots,n\}\) contain \(n_Q\) indices and write
\[
  \mathbb P_n=\frac1n\sum_{i=1}^n\delta_{W_i},
  \qquad
  \mathbb P_{\mathcal Q}
  =\frac1{n_Q}\sum_{i\in\mathcal Q}\delta_{W_i}.
\]
Support points approximate a target distribution by minimizing energy distance
\citep{mak2018support}.  Restricting representatives to observed rows gives
\begin{align}
  \mathcal E(\mathbb P_{\mathcal Q},\mathbb P_n)
  ={}&\frac{2}{n_Qn}\sum_{i\in\mathcal Q}\sum_{\ell=1}^n
      \|W_i-W_\ell\|_2
     -\frac1{n_Q^2}\sum_{i,j\in\mathcal Q}\|W_i-W_j\|_2 \notag\\
    &-\frac1{n^2}\sum_{\ell,m=1}^n\|W_\ell-W_m\|_2.
  \label{eq:supp_empirical_energy_distance}
\end{align}
Data Twinning efficiently constructs statistically similar subsets under this
criterion \citep{vakayil2022twinning}.  Using one such subset for the outer
average gives
\begin{equation}
  \widehat{\mathcal L}_{\mathrm{LL},\mathcal Q}(S)
  =
  \frac1{n_Q}\sum_{i\in\mathcal Q}g_i(S).
  \label{eq:supp_twinning_query_average}
\end{equation}
All \(n\) observations remain eligible as neighbors.  If
\(N_{\mathrm{sub}}\) subset values are evaluated, the number of local fits
falls from \(N_{\mathrm{sub}}n\) to \(N_{\mathrm{sub}}n_Q\).  The following
experiment measures the resulting approximation and runtime tradeoff.

The comparison uses Branin, Ackley--3, and Friedman functions \citep{virtuallib} with
\(p\in\{20,100\}\), \(n=500\), \(\rho\in\{0,0.5\}\), five replications, and
inverse split ratios \(r\in\{2,3,5\}\).  Selected-set agreement and runtime
are compared with the calculation using all observations.  The predictors are independent
at \(\rho=0\) and follow the spaced AR(1) design at \(\rho=0.5\).

%% file: sections/supplement_concrete_case.tex
\paragraph{Concrete eight-coordinate analysis.}
The Concrete data contain 1,030 complete mixtures \citep{yeh2007concretedata}.
The predictors are cement, blast-furnace slag, fly ash, water-to-binder ratio,
superplasticizer, coarse aggregate, fine aggregate, and age at testing. The ratio
is water divided by the sum of cement, slag, and ash; it replaces raw water.
These three binder components remain as separate coordinates. Each sample
and the full-data reference are standardized separately for PRIME, PME-kNN,
and FIRST-SHAP, with further internal rescaling disabled. SHAP uses the same
observations without standardization.
For the additional comparison of the two PME-kNN noise settings, we fit
PRIME, PME-kNN under each setting, and FIRST-SHAP to the same 100 samples
of 824 distinct rows,
using seeds \(202609260+r\), \(r=1,\ldots,100\). Each sample takes the first
824 rows of a uniform random permutation of all observations.
Each configuration is also fitted to all 1,030 observations.

PRIME uses its stated screen and local-linear PME estimator. PME-kNN uses
the original six-neighbor implementation under both noise settings on all
eight coordinates. FIRST-SHAP uses its native screen, two forward passes,
and two-neighbor Shapley estimates on retained inputs. All 400 sample fits
and four full-data fits completed.

PRIME's median water/binder allocation is 0.455, with interquartile range
\([0.438,0.478]\). The corresponding full-data estimate is 0.385.
Its full-data age allocation is 0.513. Mean coordinate interquartile widths
are 0.017 for PRIME, 0.036 for PME-kNN with \texttt{noise=FALSE},
0.096 for PME-kNN with \texttt{noise=TRUE}, and 0.021 for FIRST-SHAP.
The FALSE setting gives positive allocations to all eight coordinates.
All 100 TRUE vectors contain negative entries; their median total
is 0.708 and their full-data total is 0.729. These totals estimate explained
fractions of response variance.

\paragraph{Native PME-kNN risk contrasts.}
On the full predictor set \(P\), the \texttt{noise=TRUE} option uses the
noise-adjusted total-effect estimate \(\widehat T_P^{\mathrm{kNN}}\).
The neighborhoods vary with the conditioning subset, so the estimated differences can be
negative. Mixed signs permit cancellation in the reciprocal sums used by
the original proportional-value recursion. In the most extreme Concrete
replication, the native allocations range from \(-3.583\) to \(2.545\),
while their sum remains \(0.706\). The returned values are retained.
The PRIME and PME-NA panels in Figures~\ref{fig:friedman-methods} and~\ref{fig:concrete-p8-methods}
display \([-0.30,1.05]\), whereas FIRST-SHAP displays \([0,1.05]\).
All estimates for each displayed method are included in its box statistics.
\begin{table}[!htbp]
\centering
\caption{Native six-neighbor PME-kNN outputs under both noise settings.}
\begin{tabular}{llrrr}
\toprule
Dataset & Noise & Negative entries & Full range & Outside plot range \\
\midrule
Friedman & FALSE & 0/300 & $[0.030,\,0.319]$ & 0 \\
         & TRUE  & 116/300 & $[-2.618,\,3.771]$ & 13 \\
Concrete & FALSE & 0/800 & $[0.021,\,0.575]$ & 0 \\
         & TRUE  & 314/800 & $[-3.583,\,3.080]$ & 29 \\
\bottomrule
\end{tabular}
\label{tab:native_pme_ranges}
\end{table}

%% file: sections/supplement_results.tex
\subsection{Screening and downstream allocation}
The Friedman screening comparison in
Figure~\ref{fig:friedman_selection_illustration} is complemented by Branin
and Ackley--3 under the same three predictor laws.
\begin{figure}[!htbp]
\centering
\includegraphics[width=0.92\textwidth]{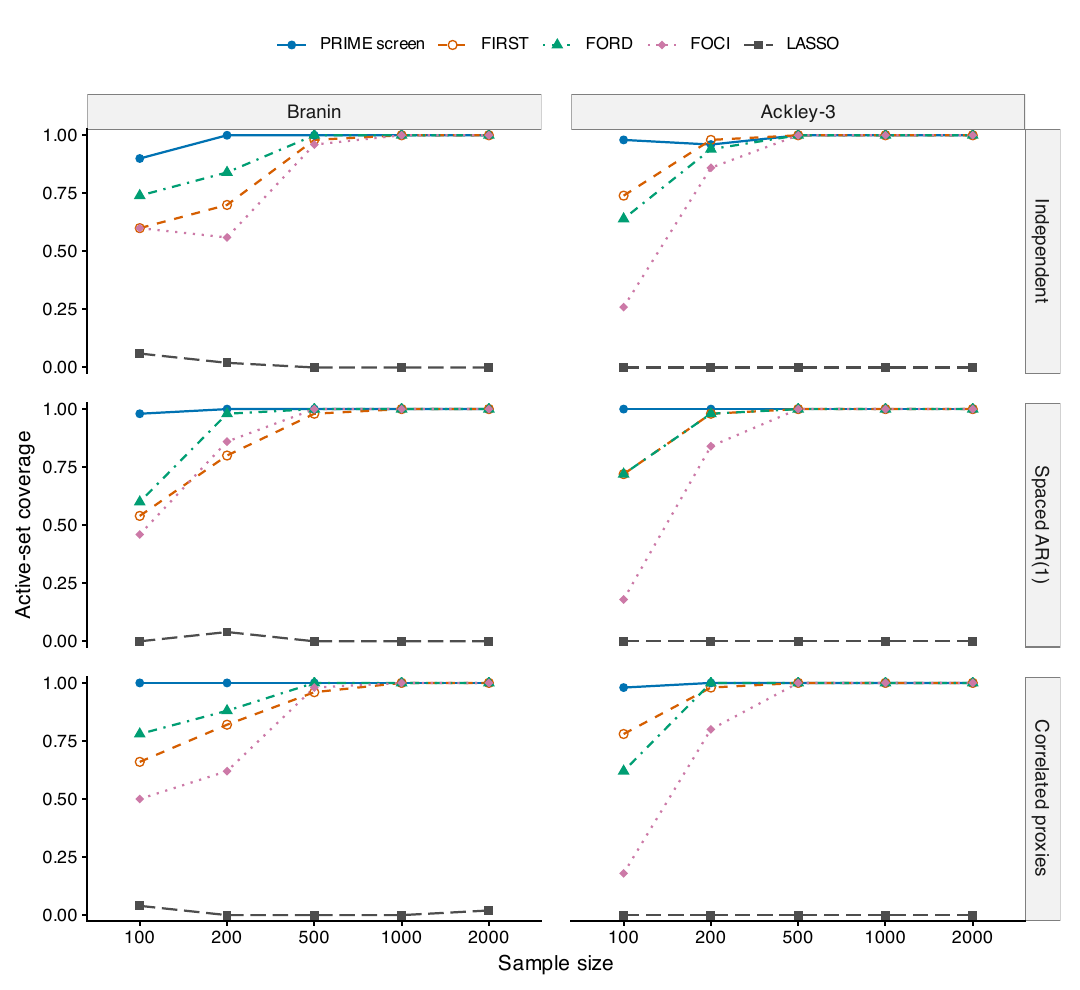}
\caption{Active-set coverage for Branin and Ackley--3 at \(p=100\), with 50
paired replications per cell. Columns identify the response and rows the
predictor law.}
\label{fig:other_functions_selection}
\end{figure}

\begin{figure}[!htbp]
\centering
\includegraphics[width=\textwidth]{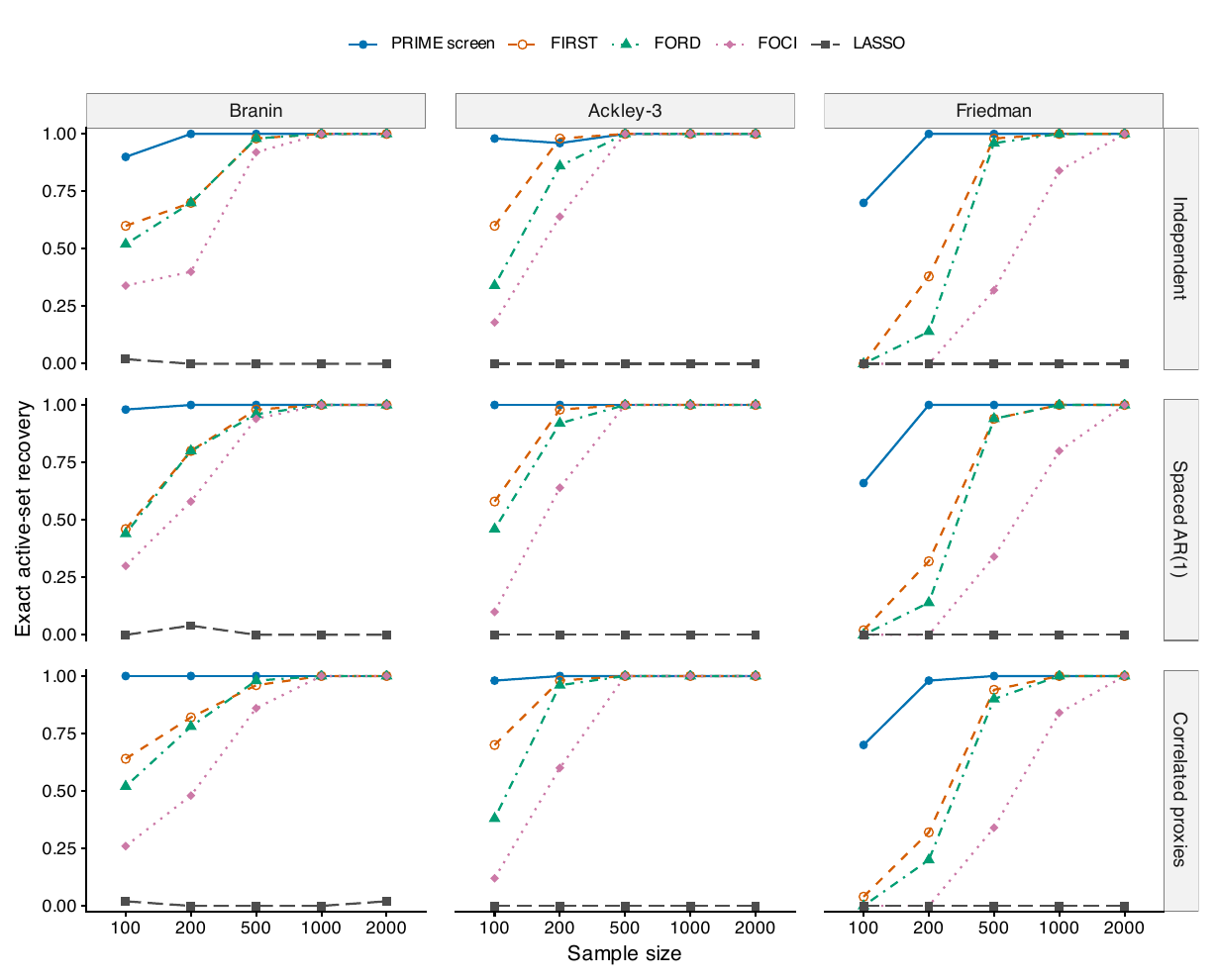}
\caption{Exact active-set recovery at \(p=100\). Each point is the proportion
of 50 paired replications with \(\widehat A=A\).}
\label{fig:supp_selection_exact}
\end{figure}

\begin{figure}[!htbp]
\centering
\includegraphics[width=\textwidth]{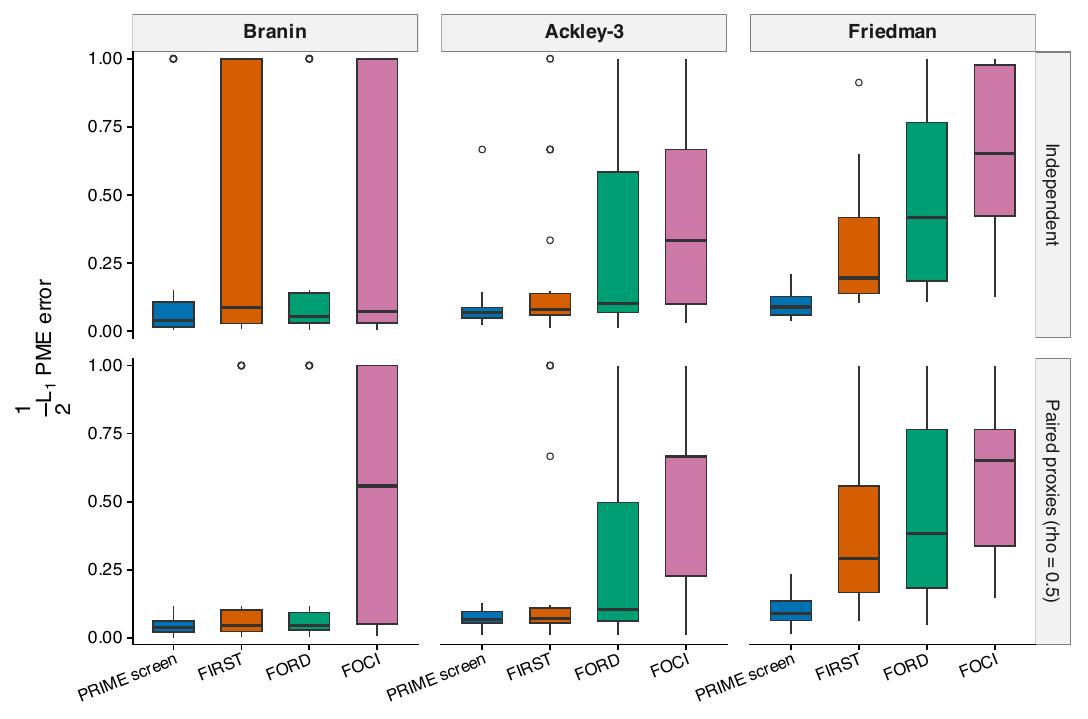}
\caption{PME error after coordinate selection at \(n=p=100\) over 30 paired
replications. Rows identify the predictor law and columns the response; boxes
use completed allocations.}
\label{fig:selection_to_pme_combined}
\end{figure}

Completion totals over the six cells are 180/180 for PRIME, 171/180 for FIRST,
172/180 for FORD, and 159/180 for FOCI. LASSO is omitted because only 74/180
allocations completed.

\FloatBarrier
\paragraph{Ridge initialization.}

\input{tables/ridge_seed_diagnostic}

The empty start has the highest coverage for Ackley--3 and Friedman,
and fixed ridge for the suppressor settings. BIC-gated initialization gives
intermediate coverage.
\paragraph{Conditional-risk estimation.}

\begin{figure}[!htbp]
\centering
\includegraphics[width=\textwidth]{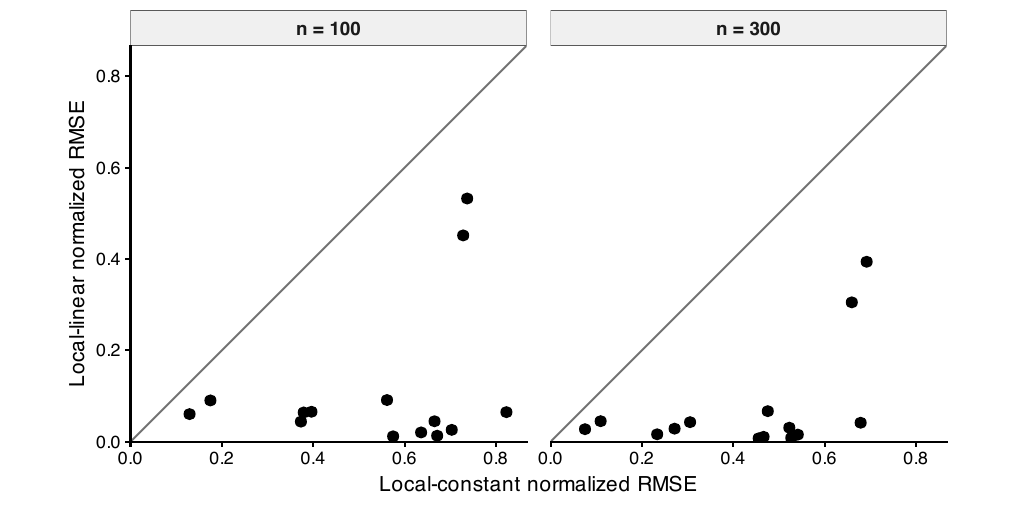}
\caption{Normalized RMSE of local-linear and local-constant risk estimates on
the known active set \(A\) over 100 paired replications. Each point is one
response--predictor-law setting; the diagonal marks equal error.}
\label{fig:supp_conditional_variance_rmse}
\end{figure}

All 28 response--sample-size--dependence combinations lie below the diagonal.
The median local-linear to local-constant RMSE ratio is 0.112, with a range
from 0.016 to 0.723.
\paragraph{PME estimator calibration.}

\begin{figure}[!htbp]
\centering
\includegraphics[width=\textwidth]{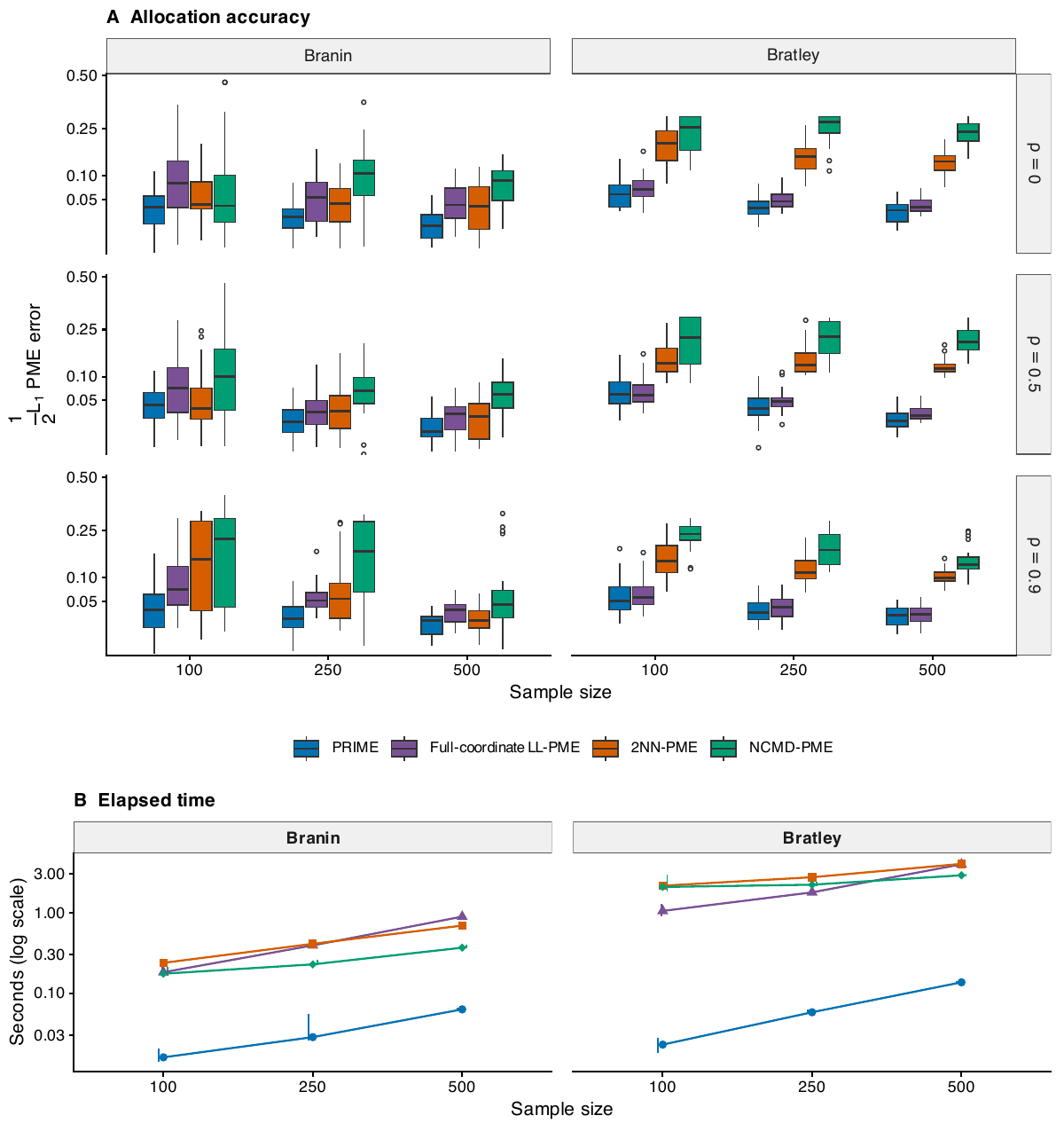}
\caption{PME estimator calibration with paired inactive proxies. Panel A shows
allocation error; Panel B shows elapsed time. Full-coordinate LL-PME uses
local-linear subset-value estimation without screening.}
\label{fig:supp_pme_low_dimensional}
\end{figure}

The full-coordinate local-linear calculation measures the effect of
screening before subset enumeration.  The oracle-coordinate calculation
measures the remaining subset-value estimation error when the active
coordinates are known. Pooled mean elapsed time is 0.06 seconds for PRIME,
compared with 1.42 for full-coordinate LL-PME, 1.70 for 2NN-PME, and 1.40 for
NCMD-PME; Figure~\ref{fig:supp_pme_low_dimensional} reports the corresponding
accuracy tradeoff.
\subsection{Correlated-proxy robustness}

Branin and Bratley provide distinct correlated-input settings: a
two-dimensional nonlinear surface and a five-coordinate nested-product
response.  Both are embedded in \(p=20\) dimensions with one inactive proxy
per active coordinate, proxy correlation \(0.9\), and
\(n\in\{100,200,500\}\).  Each selector is followed by its stated PME
estimator over 20 paired replications.

\input{tables/pme_correlated_pipeline}

Active-set coverage is the proportion of all runs whose selected set contains
every active coordinate.  Proxy allocation is the total PME assigned to the
paired inactive proxies.  Mean loss and proxy allocation use completed
allocations, while the table reports completion counts over all attempts.
PRIME has the lowest mean loss in all 6 response--sample-size cells.  The highest pooled active-set coverage is
0.758 (PRIME), and the smallest pooled proxy
allocation is 0.0012 (PRIME).
\FloatBarrier
\input{sections/supplement_allocation_rules}
\subsection{Scalability}

Five response functions are evaluated at \(n=500\) and
\(p\in\{20,50,100,250,500\}\).  The active dimension remains fixed while
inactive coordinates are added. Figure~\ref{fig:pme_scaling} displays three of
the five responses; the numerical summaries use the full five-response grid.
Because the response and the active-input law are fixed across \(p\), this
experiment measures allocation error as inactive coordinates are added while
the intrinsic regression dimension remains unchanged.
This comparison also includes nearest-neighbor variable screening (NNVS)
\citep{chatterjee2026ncmd}, followed by 2NN-PME.

\begin{figure}[!htbp]
\centering
\includegraphics[width=\textwidth]{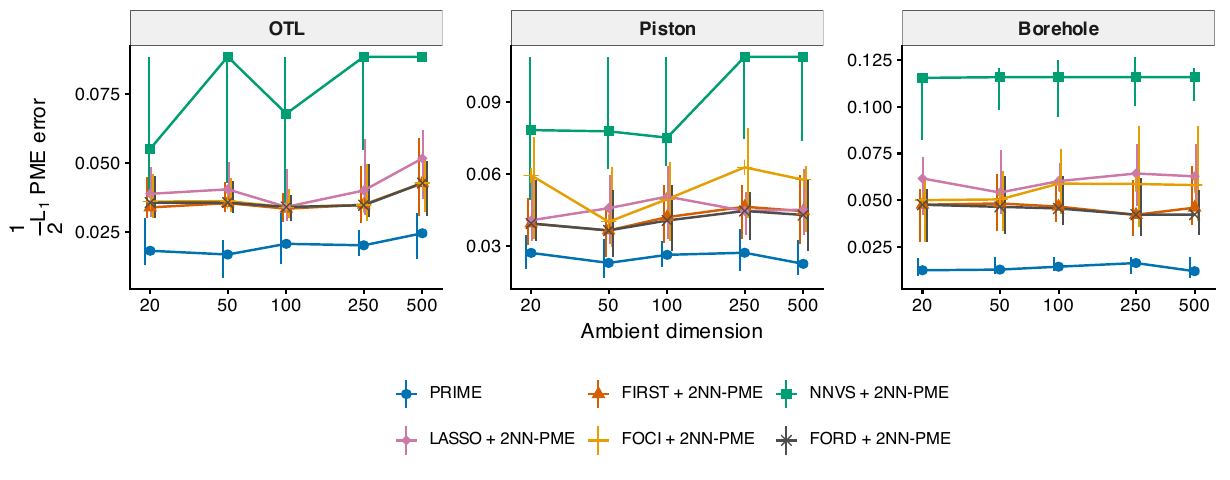}
\caption{PME error as inactive coordinates accumulate at \(n=500\). Points
and bars show medians and interquartile ranges over completed runs.}
\label{fig:pme_scaling}
\end{figure}

PRIME has the lowest mean error in all 25 response--dimension cells.
Its pooled mean error over completed runs is 0.026, compared with
0.047 for FORD, 0.050 for FIRST, 0.065 for FOCI, 0.088 for NNVS, and 0.138 for LASSO, each followed by 2NN-PME.
Comparator completion counts are 750/750 for FIRST, 740/750 for NNVS,
723/750 for LASSO, 749/750 for FOCI, and 750/750 for FORD.

Across the six methods, no OTL, piston, or borehole run achieves complete
active-set coverage.
Allocation accuracy and complete active-set coverage are therefore reported
as distinct empirical endpoints.

At the largest ambient dimension, Table~\ref{tab:pme-scaling-pmax}
gives detailed runtimes, selected dimensions, allocation losses, and
completion counts. The true active dimensions of OTL, piston, and borehole are
6, 7, and 8; the selected sizes are therefore interpreted jointly with loss,
not as exact support recovery.

\input{tables/pme_scaling_pmax}
\paragraph{Acceleration using representative observations.}

\input{tables/twinning_diagnostic}

The median selected-set Jaccard index is one for every split ratio.
Median runtime speedup increases from 1.64 at ratio 2 to 3.52 at ratio 5.
This diagnostic reduces the number of local fits: all observations remain eligible as
neighbors, and the table does not establish an ANN, memory, or final-allocation
approximation guarantee.

\FloatBarrier
\subsection{Additional observed-data checks}
A separate common protocol provides a robustness check on Concrete, Airfoil,
and Auto MPG \citep{yeh2007concretedata,brooks1989airfoil,quinlan1993autompg}.
It uses 200 paired nonparametric bootstrap samples and five repeats of
five-fold nested cross-validation, comparing PRIME with full-coordinate
LL-PME.
\ifConcreteMain
For Concrete, these bootstrap draws differ from the 100 distinct-row
subsamples in the main comparison; the held-out prediction check adds a
separate assessment.
\else
This protocol is distinct from the paired 80\% Auto MPG analysis in the main
paper.
\fi

\begin{figure}[!htbp]
\centering
\includegraphics[width=\textwidth]{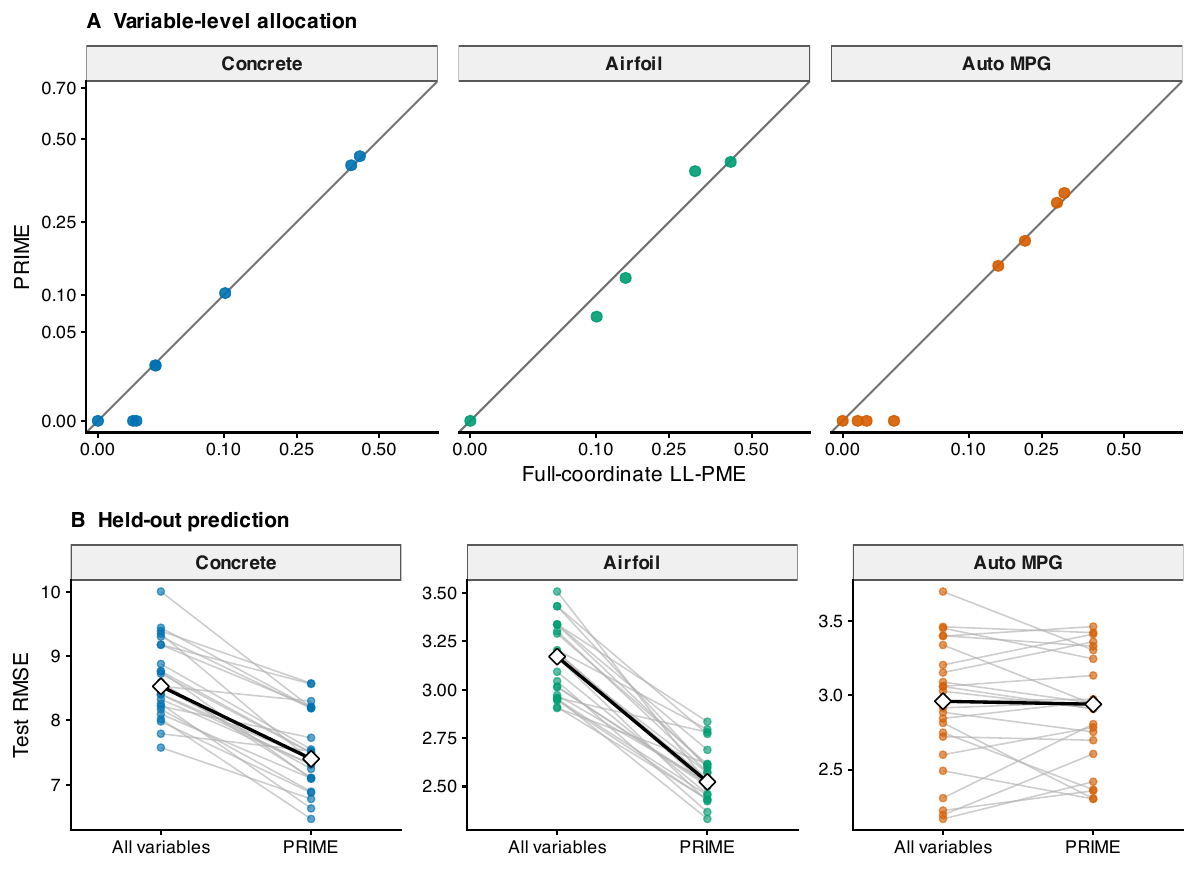}
\caption{Observed-data diagnostics: Panel A compares bootstrap-median PRIME
and full-coordinate LL-PME allocations; Panel B compares paired held-out RMSE
for full and screened predictor sets.}
\label{fig:observed-application}
\end{figure}

\input{tables/observed_application_main}

\Needspace{8\baselineskip}
Median screened-to-full allocation distances are 0.036 for Concrete, 0.087 for Airfoil, and 0.055 for Auto MPG.
All 200 bootstrap pairs are complete for each dataset.  Screening reduces median subset-value counts from 256 to 32 for Concrete, from 32 to 8 for Airfoil, and from 256 to 32 for Auto MPG.
The median paired selected-to-full test-RMSE ratios are 0.868 for Concrete, 0.806 for Airfoil, and 0.991 for Auto MPG.
These checks show that the reduced representations retain comparable held-out
prediction while requiring fewer coalition values.

\FloatBarrier

%% file: tables/ridge_seed_diagnostic.tex
\begin{table}[H]
\centering
\caption{Active-set coverage under alternative forward initializations.}
\label{tab:ridge-seed-diagnostic}
\begin{tabular}{lccc}
\toprule
Scenario & Fixed ridge & BIC-gated & Empty
\\
\midrule
Ackley--3, $n=p=100$ & 0.067 & 0.967 & 1.000 \\
Friedman, $n=p=100$ & 0.233 & 0.767 & 0.900 \\
Suppressor, $n=100$, $\sigma=0.1$ & 0.780 & 0.720 & 0.040 \\
Suppressor, $n=500$, $\sigma=0.5$ & 0.780 & 0.680 & 0.160 \\
\bottomrule
\end{tabular}
\end{table}

%% file: tables/pme_correlated_pipeline.tex
\begin{table}[H]
\centering
\caption{Correlated-proxy performance pooled over Branin and Bratley.}
\label{tab:correlated-pipeline}
\small
\begin{tabular}{lrrrr}
\toprule
Method & Mean loss & Active-set coverage & Proxy allocation & Completed \\
\midrule
PRIME & 0.040 & 0.758 & 0.0012 & 120/120 \\
FIRST + 2NN-PME & 0.103 & 0.458 & 0.0073 & 120/120 \\
NNVS + 2NN-PME & 0.078 & 0.492 & 0.0129 & 120/120 \\
LASSO + 2NN-PME & 0.213 & 0.225 & 0.0083 & 93/120 \\
FOCI + 2NN-PME & 0.102 & 0.450 & 0.0103 & 118/120 \\
FORD + 2NN-PME & 0.086 & 0.483 & 0.0053 & 120/120 \\
\bottomrule
\end{tabular}
\end{table}

%% file: sections/supplement_allocation_rules.tex
\subsection{Allocation rules and estimator sensitivity}
\label{sec:supp-allocation-rules}
For the ten-coordinate Friedman design in the main text, we estimate
normalized, monotone-projected subset values using local-linear fits.
We then compute PME and Shapley from the same estimates, so the comparison
changes only the allocation rule.

\begin{figure}[!htbp]
\centering
\includegraphics[width=0.92\textwidth]{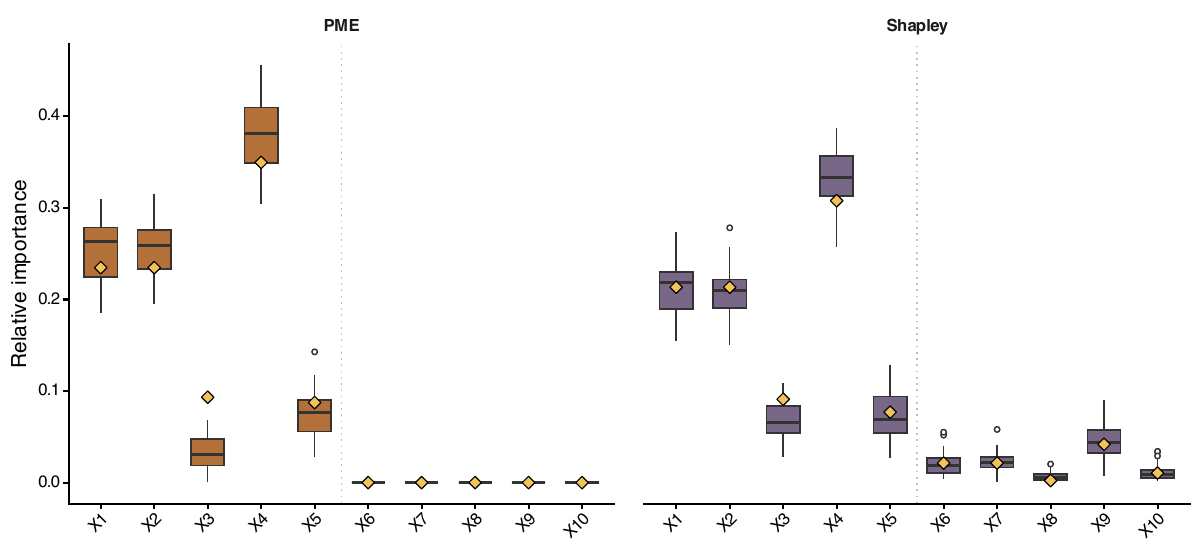}
\caption{PME and Shapley computed from the same full-coordinate local-linear
subset estimates at \(n=200\) and \(\rho=0.5\). Diamonds mark population
values; the dotted line separates active inputs and proxies.}
\label{fig:friedman-ll-rules}
\end{figure}

The mean total allocation to the five proxies is 0.00014 for PME and 0.10749
for Shapley, compared with population values zero and 0.09728. The distinction
persists across sample sizes and strengthens as the proxy correlation grows
(Figure~\ref{fig:friedman-dependence}).

\begin{figure}[!htbp]
\centering
\includegraphics[width=\textwidth]{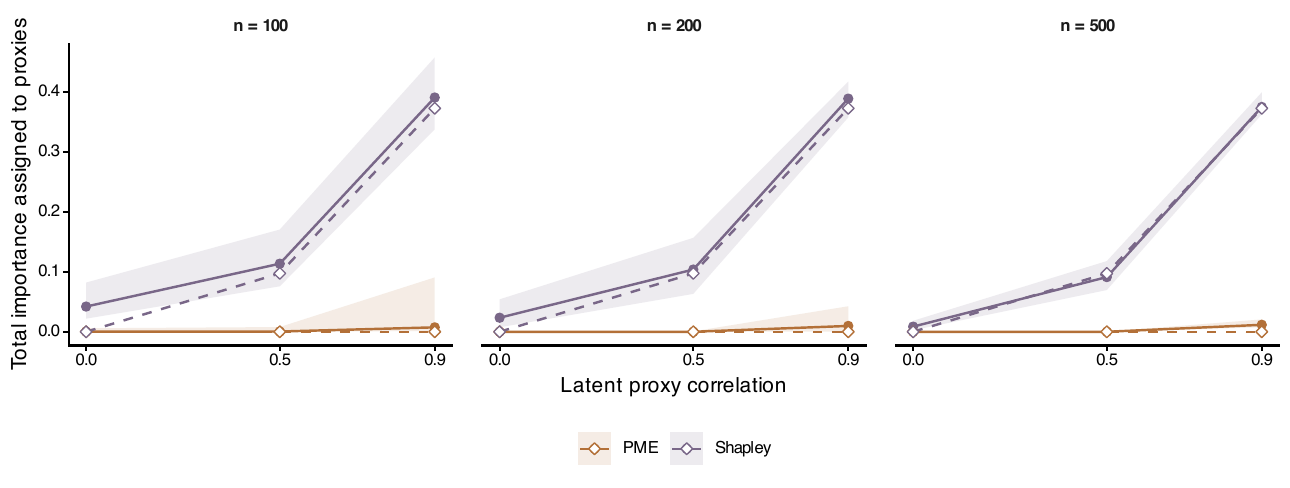}
\caption{Total proxy allocation as dependence increases. Points and bands show
the median and 5th--95th percentiles over 30 datasets; dashed lines and open
diamonds show population values.}
\label{fig:friedman-dependence}
\end{figure}

At $\rho=0.9$, the population Shapley proxy mass is 0.37224, whereas the
population PME remains zero. Finite-sample PME has residual mean proxy mass
0.02341, 0.01432, and 0.01249 as $n$ increases from 100 to 500. The
corresponding Shapley means are 0.39323, 0.38839, and 0.37725.

\paragraph{Estimator sensitivity.}
Matched-\(K\) and two-neighbor local-constant estimates provide further
allocation-rule comparisons on the same datasets. Within each estimator, PME
and Shapley use the same monotone-projected subset values.

\begin{table}[!htbp]
\centering
\caption{Estimator sensitivity at \(n=200\) and \(\rho=0.5\) over 30 paired
Friedman datasets.}
\begin{tabular}{llcc}
\toprule
Rule & Game estimator & Mean error & Mean proxy mass \\
\midrule
PME & Full local-linear & 0.09379 & 0.00014 \\
PME & Matched-\(K\) local-constant & 0.44029 & 0.00082 \\
PME & Two-neighbor local-constant & 0.16836 & 0.00056 \\
Shapley & Full local-linear & 0.08641 & 0.10749 \\
Shapley & Matched-\(K\) local-constant & 0.17978 & 0.11888 \\
Shapley & Two-neighbor local-constant & 0.10243 & 0.13831 \\
\bottomrule
\end{tabular}
\label{tab:allocation-estimator-sensitivity}
\end{table}

Holding the neighborhood rule fixed, local linearization reduces mean error
from 0.44029 to 0.09379 for PME and from 0.17978 to 0.08641 for Shapley. Full
LL-PME has lower mean PME error than both local-constant estimators in all nine
\((n,\rho)\) cells. The large-neighborhood local-constant estimator is also
less accurate than 2NN, consistent with conditional-mean drift accumulating
over a wider neighborhood.

\paragraph{Projection sensitivity.}
For the two-neighbor local-constant estimator, we compare Shapley before and
after monotone projection. At \(n=200\) and \(\rho=0.5\), mean total proxy
allocation is 0.02922 before projection and 0.13831 after projection.
Table~\ref{tab:projection-sensitivity} reports all conditions.
\input{tables/projection_sensitivity}
\FloatBarrier

%% file: tables/projection_sensitivity.tex
\begin{table}[!ht]
\centering
\caption{Shapley proxy allocation before and after monotone projection of the
controlled two-neighbor game, over 30 datasets per condition.}
\begin{tabular}{rrcc}
\toprule
$n$ & $\rho$ & Raw game & Projected game \\
\midrule
100 & 0 & -0.16239 & 0.08757 \\
100 & 0.5 & -0.03505 & 0.14592 \\
100 & 0.9 & 0.36126 & 0.40458 \\
200 & 0 & -0.12646 & 0.05681 \\
200 & 0.5 & 0.02922 & 0.13831 \\
200 & 0.9 & 0.36040 & 0.39743 \\
500 & 0 & -0.11629 & 0.03215 \\
500 & 0.5 & 0.01926 & 0.11614 \\
500 & 0.9 & 0.35986 & 0.39185 \\
\bottomrule
\end{tabular}
\label{tab:projection-sensitivity}
\end{table}

%% file: tables/pme_scaling_pmax.tex
\begin{table}[H]
\centering
\caption{Screen--reduce--allocate pipelines at \(n=p=500\) over 30 paired
runs.}
\label{tab:pme-scaling-pmax}
\footnotesize
\setlength{\tabcolsep}{7pt}
\begin{tabular}{llrrrr}
\toprule
Response & Method & Median loss & Total time (s) & Selected $|\widehat A|$ & Completed \\
\midrule
OTL & PRIME & 0.024 & 5.83 & 4.0 & 30/30 \\
 & FIRST + 2NN-PME & 0.043 & 2.89 & 3.0 & 30/30 \\
 & NNVS + 2NN-PME & 0.089 & 1.58 & 3.0 & 29/30 \\
 & LASSO + 2NN-PME & 0.052 & 0.53 & 5.0 & 30/30 \\
 & FOCI + 2NN-PME & 0.043 & 2.96 & 4.0 & 30/30 \\
 & FORD + 2NN-PME & 0.043 & 7.11 & 3.0 & 30/30 \\
\addlinespace
Piston & PRIME & 0.023 & 7.36 & 4.0 & 30/30 \\
 & FIRST + 2NN-PME & 0.045 & 3.12 & 4.0 & 30/30 \\
 & NNVS + 2NN-PME & 0.109 & 1.60 & 3.0 & 30/30 \\
 & LASSO + 2NN-PME & 0.045 & 0.57 & 4.0 & 30/30 \\
 & FOCI + 2NN-PME & 0.058 & 2.99 & 4.0 & 30/30 \\
 & FORD + 2NN-PME & 0.043 & 8.73 & 4.0 & 30/30 \\
\addlinespace
Borehole & PRIME & 0.012 & 9.65 & 5.0 & 30/30 \\
 & FIRST + 2NN-PME & 0.046 & 2.92 & 4.0 & 30/30 \\
 & NNVS + 2NN-PME & 0.116 & 0.99 & 2.0 & 30/30 \\
 & LASSO + 2NN-PME & 0.063 & 0.57 & 5.0 & 30/30 \\
 & FOCI + 2NN-PME & 0.058 & 3.02 & 4.0 & 30/30 \\
 & FORD + 2NN-PME & 0.042 & 8.58 & 4.0 & 30/30 \\
\bottomrule
\end{tabular}
\end{table}

%% file: tables/twinning_diagnostic.tex
\begin{table}[H]
\centering
\caption{Approximation using representative observations at \(n=500\) over 60 paired
datasets per split ratio.}
\label{tab:twinning-diagnostic}
\small
\begin{tabular}{rrrrr}
\toprule
Split ratio $r$ & Observations $n_Q$ & Selected-set Jaccard & Active recall & Speedup \\
\midrule
2 & 250 & 1.000 & 1.000 & 1.64 [1.48, 1.72] \\
3 & 167 & 1.000 & 1.000 & 2.29 [2.03, 2.44] \\
5 & 100 & 1.000 & 1.000 & 3.52 [2.96, 3.78] \\
\bottomrule
\end{tabular}
\end{table}

%% file: tables/observed_application_main.tex
\begin{table}[H]
\centering
\caption{Observed-data comparison under bootstrap allocation and nested
cross-validation.}
\label{tab:observed-application}
\small
\begin{tabular}{lrrrrrr}
\toprule
Data & $n$ & $p$ & Selected $d$ & \multicolumn{2}{c}{Subset-value evaluations} & Test-RMSE ratio \\
\cmidrule(lr){5-6}
 & & & & PRIME & Full-coordinate & \\
\midrule
Concrete & 1030 & 8 & 5 & 32 & 256 & 0.868 \\
Airfoil & 1503 & 5 & 3 & 8 & 32 & 0.806 \\
Auto MPG & 392 & 8 & 5 & 32 & 256 & 0.991 \\
\bottomrule
\end{tabular}
\end{table}

%% file: main.bbl
\begin{thebibliography}{}

\bibitem[\protect\citeauthoryear{Ackley}{Ackley}{1987}]{ackley1987connectionist}
Ackley, D.~H. (1987).
\newblock {\em A Connectionist Machine for Genetic Hillclimbing}.
\newblock Boston: Kluwer Academic Publishers.

\bibitem[\protect\citeauthoryear{Azadkia and Chatterjee}{Azadkia and
  Chatterjee}{2021}]{azadkia2021simple}
Azadkia, M. and S.~Chatterjee (2021).
\newblock A simple measure of conditional dependence.
\newblock {\em The Annals of Statistics\/}~{\em 49\/}(6), 3070--3102.

\bibitem[\protect\citeauthoryear{Azadkia and Roudaki}{Azadkia and
  Roudaki}{2025}]{azadkia2025integrated}
Azadkia, M. and P.~Roudaki (2025).
\newblock A new measure of dependence: Integrated {$R^2$}.
\newblock \href{https://arxiv.org/abs/2505.18146}{arXiv:2505.18146}.

\bibitem[\protect\citeauthoryear{Borgonovo, Plischke, and Rabitti}{Borgonovo
  et~al.}{2024}]{borgonovo2024many}
Borgonovo, E., E.~Plischke, and G.~Rabitti (2024).
\newblock The many shapley values for explainable artificial intelligence: A
  sensitivity analysis perspective.
\newblock {\em European Journal of Operational Research\/}~{\em 318\/}(3),
  911--926.

\bibitem[\protect\citeauthoryear{Brooks, Pope, and Marcolini}{Brooks
  et~al.}{1989}]{brooks1989airfoil}
Brooks, T., D.~Pope, and M.~Marcolini (1989).
\newblock Airfoil self-noise.
\newblock UCI Machine Learning Repository.
\newblock \url{https://doi.org/10.24432/C5VW2C}.

\bibitem[\protect\citeauthoryear{Broto, Bachoc, and Depecker}{Broto
  et~al.}{2020}]{broto2020variance}
Broto, B., F.~Bachoc, and M.~Depecker (2020).
\newblock Variance reduction for estimation of {S}hapley effects and adaptation
  to unknown input distribution.
\newblock {\em SIAM/ASA Journal on Uncertainty Quantification\/}~{\em 8\/}(2),
  693--716.

\bibitem[\protect\citeauthoryear{Chatterjee, Niu, and Bhattacharya}{Chatterjee
  et~al.}{2026}]{chatterjee2026ncmd}
Chatterjee, A., Z.~Niu, and B.~B. Bhattacharya (2026).
\newblock Conditional mean independence and global sensitivity analysis using
  nearest neighbor graphs.
\newblock \href{https://arxiv.org/abs/2607.04692v2}{arXiv:2607.04692v2}.

\bibitem[\protect\citeauthoryear{Covert, Lundberg, and Lee}{Covert
  et~al.}{2020}]{covert2020understanding}
Covert, I., S.~M. Lundberg, and S.-I. Lee (2020).
\newblock Understanding global feature contributions with additive importance
  measures.
\newblock In {\em Advances in Neural Information Processing Systems},
  Volume~33, pp.\  17212--17223. Curran Associates, Inc.

\bibitem[\protect\citeauthoryear{Da~Veiga, Gamboa, Iooss, and Prieur}{Da~Veiga
  et~al.}{2021}]{sensitivitybook2021}
Da~Veiga, S., F.~Gamboa, B.~Iooss, and C.~Prieur (2021).
\newblock {\em Basics and Trends in Sensitivity Analysis: Theory and Practice
  in R}.
\newblock SIAM.

\bibitem[\protect\citeauthoryear{Donoho and Johnstone}{Donoho and
  Johnstone}{1994}]{donoho1994ideal}
Donoho, D.~L. and I.~M. Johnstone (1994).
\newblock Ideal spatial adaptation by wavelet shrinkage.
\newblock {\em Biometrika\/}~{\em 81\/}(3), 425--455.

\bibitem[\protect\citeauthoryear{Fan and Gijbels}{Fan and
  Gijbels}{1996}]{fan1996local}
Fan, J. and I.~Gijbels (1996).
\newblock {\em Local Polynomial Modelling and Its Applications}.
\newblock London: Chapman and Hall.

\bibitem[\protect\citeauthoryear{Feldman}{Feldman}{1999}]{feldman1999proportional}
Feldman, B.~E. (1999).
\newblock The proportional value of a cooperative game.
\newblock Manuscript, Scudder Kemper Investments, Chicago.
  \url{http://fmwww.bc.edu/RePEc/es2000/1140.pdf}.

\bibitem[\protect\citeauthoryear{Feldman}{Feldman}{2005}]{feldman2005relative}
Feldman, B.~E. (2005).
\newblock Relative importance and value.
\newblock SSRN working paper. \url{https://doi.org/10.2139/ssrn.2255827}.

\bibitem[\protect\citeauthoryear{Fisher, Rudin, and Dominici}{Fisher
  et~al.}{2019}]{fisher2019all}
Fisher, A., C.~Rudin, and F.~Dominici (2019).
\newblock All models are wrong, but many are useful: Learning a variable's
  importance by studying an entire class of prediction models simultaneously.
\newblock {\em Journal of Machine Learning Research\/}~{\em 20\/}(177), 1--81.

\bibitem[\protect\citeauthoryear{Friedman}{Friedman}{1991}]{friedman1991multivariate}
Friedman, J.~H. (1991).
\newblock Multivariate adaptive regression splines.
\newblock {\em The Annals of Statistics\/}~{\em 19\/}(1), 1--67.

\bibitem[\protect\citeauthoryear{Gr{\"o}mping}{Gr{\"o}mping}{2007}]{gromping2007importance}
Gr{\"o}mping, U. (2007).
\newblock Estimators of relative importance in linear regression based on
  variance decomposition.
\newblock {\em The American Statistician\/}~{\em 61\/}(2), 139--147.

\bibitem[\protect\citeauthoryear{Hart and Gremaud}{Hart and
  Gremaud}{2018}]{hart2018approximation}
Hart, J.~L. and P.~A. Gremaud (2018).
\newblock An approximation theoretic perspective of {Sobol'} indices with
  dependent variables.
\newblock {\em International Journal for Uncertainty Quantification\/}~{\em
  8\/}(6), 483--493.

\bibitem[\protect\citeauthoryear{Herin, Il~Idrissi, Chabridon, and Iooss}{Herin
  et~al.}{2024}]{herin2024proportional}
Herin, M., M.~Il~Idrissi, V.~Chabridon, and B.~Iooss (2024).
\newblock Proportional marginal effects for global sensitivity analysis.
\newblock {\em SIAM/ASA Journal on Uncertainty Quantification\/}~{\em 12\/}(2),
  667--692.

\bibitem[\protect\citeauthoryear{Homma and Saltelli}{Homma and
  Saltelli}{1996}]{homma1996importance}
Homma, T. and A.~Saltelli (1996).
\newblock Importance measures in global sensitivity analysis of nonlinear
  models.
\newblock {\em Reliability Engineering \& System Safety\/}~{\em 52\/}(1),
  1--17.

\bibitem[\protect\citeauthoryear{Huang and Joseph}{Huang and
  Joseph}{2025}]{huang2025first}
Huang, C. and V.~R. Joseph (2025).
\newblock Factor importance ranking and selection using total indices.
\newblock {\em Technometrics\/}~{\em 67\/}(4), 573--589.

\bibitem[\protect\citeauthoryear{Iooss, Da~Veiga, Janon, Pujol, et~al.}{Iooss
  et~al.}{2022}]{iooss2022sensitivity}
Iooss, B., S.~Da~Veiga, A.~Janon, G.~Pujol, et~al. (2022).
\newblock {sensitivity}: Global sensitivity analysis of model outputs.
\newblock R package version 1.27.1.
  \url{https://CRAN.R-project.org/package=sensitivity}.

\bibitem[\protect\citeauthoryear{Iooss and Prieur}{Iooss and
  Prieur}{2019}]{iooss2019shapley}
Iooss, B. and C.~Prieur (2019).
\newblock Shapley effects for sensitivity analysis with correlated inputs:
  Comparisons with {S}obol' indices, numerical estimation and applications.
\newblock {\em International Journal for Uncertainty Quantification\/}~{\em
  9\/}(5), 493--514.

\bibitem[\protect\citeauthoryear{Lei, G'Sell, Rinaldo, Tibshirani, and
  Wasserman}{Lei et~al.}{2018}]{lei2018distribution}
Lei, J., M.~G'Sell, A.~Rinaldo, R.~J. Tibshirani, and L.~Wasserman (2018).
\newblock Distribution-free predictive inference for regression.
\newblock {\em Journal of the American Statistical Association\/}~{\em
  113\/}(523), 1094--1111.

\bibitem[\protect\citeauthoryear{Lindeman, Merenda, and Gold}{Lindeman
  et~al.}{1980}]{lindeman1980introduction}
Lindeman, R.~H., P.~F. Merenda, and R.~Z. Gold (1980).
\newblock {\em Introduction to bivariate and multivariate analysis}.
\newblock Glenview, IL: Scott, Foresman.

\bibitem[\protect\citeauthoryear{Lundberg and Lee}{Lundberg and
  Lee}{2017}]{lundberg2017unified}
Lundberg, S.~M. and S.-I. Lee (2017).
\newblock A unified approach to interpreting model predictions.
\newblock In {\em Advances in Neural Information Processing Systems},
  Volume~30. Curran Associates, Inc.

\bibitem[\protect\citeauthoryear{Mak and Joseph}{Mak and
  Joseph}{2018}]{mak2018support}
Mak, S. and V.~R. Joseph (2018).
\newblock Support points.
\newblock {\em The Annals of Statistics\/}~{\em 46\/}(6A), 2562--2592.

\bibitem[\protect\citeauthoryear{Ortmann}{Ortmann}{2000}]{ortmann2000proportional}
Ortmann, K.~M. (2000).
\newblock The proportional value for positive cooperative games.
\newblock {\em Mathematical Methods of Operations Research\/}~{\em 51\/}(2),
  235--248.

\bibitem[\protect\citeauthoryear{Owen}{Owen}{2014}]{owen2014sobol}
Owen, A.~B. (2014).
\newblock {S}obol' indices and {S}hapley value.
\newblock {\em SIAM/ASA Journal on Uncertainty Quantification\/}~{\em 2\/}(1),
  245--251.

\bibitem[\protect\citeauthoryear{Owen and Prieur}{Owen and
  Prieur}{2017}]{owen2017shapley}
Owen, A.~B. and C.~Prieur (2017).
\newblock On {S}hapley value for measuring importance of dependent inputs.
\newblock {\em SIAM/ASA Journal on Uncertainty Quantification\/}~{\em 5\/}(1),
  986--1002.

\bibitem[\protect\citeauthoryear{Quinlan}{Quinlan}{1993}]{quinlan1993autompg}
Quinlan, R. (1993).
\newblock Auto {MPG}.
\newblock UCI Machine Learning Repository.
\newblock \url{https://doi.org/10.24432/C5859H}.

\bibitem[\protect\citeauthoryear{Ruppert, Wand, and Carroll}{Ruppert
  et~al.}{2003}]{ruppert2003semiparametric}
Ruppert, D., M.~P. Wand, and R.~J. Carroll (2003).
\newblock {\em Semiparametric Regression}.
\newblock Cambridge: Cambridge University Press.

\bibitem[\protect\citeauthoryear{Saltelli, Ratto, Andres, Campolongo, Cariboni,
  Gatelli, Saisana, and Tarantola}{Saltelli et~al.}{2008}]{saltelli2008global}
Saltelli, A., M.~Ratto, T.~Andres, F.~Campolongo, J.~Cariboni, D.~Gatelli,
  M.~Saisana, and S.~Tarantola (2008).
\newblock {\em Global Sensitivity Analysis: the Primer}.
\newblock John Wiley \& Sons.

\bibitem[\protect\citeauthoryear{Schwarz}{Schwarz}{1978}]{schwarz1978estimating}
Schwarz, G. (1978).
\newblock Estimating the dimension of a model.
\newblock {\em The Annals of Statistics\/}~{\em 6\/}(2), 461--464.

\bibitem[\protect\citeauthoryear{Shapley}{Shapley}{1953}]{shapley1953shap}
Shapley, L.~S. (1953).
\newblock A value for n-person games.
\newblock In H.~W. Kuhn and A.~W. Tucker (Eds.), {\em Contributions to the
  Theory of Games, Volume II}, pp.\  307--318. Princeton: Princeton University
  Press.

\bibitem[\protect\citeauthoryear{Sobol'}{Sobol'}{1993}]{sobol1993sensitivity}
Sobol', I.~M. (1993).
\newblock Sensitivity estimates for nonlinear mathematical models.
\newblock {\em Mathematical Modeling and Computational Experiment\/}~{\em
  1\/}(4), 407--414.

\bibitem[\protect\citeauthoryear{Song, Nelson, and Staum}{Song
  et~al.}{2016}]{song2016shapley}
Song, E., B.~L. Nelson, and J.~Staum (2016).
\newblock {S}hapley effects for global sensitivity analysis: Theory and
  computation.
\newblock {\em SIAM/ASA Journal on Uncertainty Quantification\/}~{\em 4\/}(1),
  1060--1083.

\bibitem[\protect\citeauthoryear{Surjanovic and Bingham}{Surjanovic and
  Bingham}{2013}]{virtuallib}
Surjanovic, S. and D.~Bingham (2013).
\newblock {\em Virtual Library of Simulation Experiments: Test Functions and
  Datasets}.
\newblock Retrieved July 11, 2025.

\bibitem[\protect\citeauthoryear{Tibshirani}{Tibshirani}{1996}]{tibshirani1996regression}
Tibshirani, R. (1996).
\newblock Regression shrinkage and selection via the {LASSO}.
\newblock {\em Journal of the Royal Statistical Society: Series B
  (Methodological)\/}~{\em 58\/}(1), 267--288.

\bibitem[\protect\citeauthoryear{Vakayil and Joseph}{Vakayil and
  Joseph}{2022}]{vakayil2022twinning}
Vakayil, A. and V.~R. Joseph (2022).
\newblock Data twinning.
\newblock {\em Statistical Analysis and Data Mining: The ASA Data Science
  Journal\/}~{\em 15\/}(5), 598--610.

\bibitem[\protect\citeauthoryear{Verdinelli and Wasserman}{Verdinelli and
  Wasserman}{2024a}]{verdinelli2024decorrelated}
Verdinelli, I. and L.~Wasserman (2024a).
\newblock Decorrelated variable importance.
\newblock {\em Journal of Machine Learning Research\/}~{\em 25\/}(7), 1--27.

\bibitem[\protect\citeauthoryear{Verdinelli and Wasserman}{Verdinelli and
  Wasserman}{2024b}]{verdinelli2024feature}
Verdinelli, I. and L.~Wasserman (2024b).
\newblock Feature importance: A closer look at shapley values and {LOCO}.
\newblock {\em Statistical Science\/}~{\em 39\/}(4), 623--636.

\bibitem[\protect\citeauthoryear{Williamson and Feng}{Williamson and
  Feng}{2020}]{williamson2020spvim}
Williamson, B. and J.~Feng (2020).
\newblock Efficient nonparametric statistical inference on population feature
  importance using shapley values.
\newblock In {\em Proceedings of the 37th International Conference on Machine
  Learning}, Volume 119 of {\em Proceedings of Machine Learning Research}, pp.\
   10282--10291. PMLR.

\bibitem[\protect\citeauthoryear{Williamson, Gilbert, Simon, and
  Carone}{Williamson et~al.}{2023}]{williamson2023framework}
Williamson, B.~D., P.~B. Gilbert, N.~R. Simon, and M.~Carone (2023).
\newblock A general framework for inference on algorithm-agnostic variable
  importance.
\newblock {\em Journal of the American Statistical Association\/}~{\em
  118\/}(543), 1645--1658.

\bibitem[\protect\citeauthoryear{Yeh}{Yeh}{1998a}]{yeh2007concretedata}
Yeh, I.-C. (1998a).
\newblock Concrete compressive strength.
\newblock UCI Machine Learning Repository.
\newblock \url{https://doi.org/10.24432/C5PK67}.

\bibitem[\protect\citeauthoryear{Yeh}{Yeh}{1998b}]{yeh1998concrete}
Yeh, I.-C. (1998b).
\newblock Modeling of strength of high-performance concrete using artificial
  neural networks.
\newblock {\em Cement and Concrete Research\/}~{\em 28\/}(12), 1797--1808.

\end{thebibliography}
